\documentclass[lettersize,journal]{IEEEtran}
\IEEEoverridecommandlockouts
\usepackage{array}
\usepackage[caption=false,font=normalsize,labelfont=sf,textfont=sf]{subfig}
\usepackage{stfloats}
\usepackage{url}
\usepackage{verbatim}
\usepackage{graphicx} % Required for inserting images
\usepackage{amsmath}
\usepackage{mathtools}
\usepackage{cite}
\usepackage{amsfonts}
\usepackage{amsthm}
\usepackage{amssymb}
\usepackage{algorithm}
\usepackage{algpseudocode}
\usepackage{textcomp}
\usepackage{xcolor}
\usepackage{multirow}
\usepackage{subcaption}
\usepackage{placeins} % Include the package at the beginning
\usepackage{bm}
\usepackage{bbm}
\usepackage{hyperref}
\usepackage[capitalise]{cleveref}
\usepackage{makecell}
\usepackage{pgfplots}
\usepgfplotslibrary{groupplots}
\pgfplotsset{compat=newest}
\usepackage[inline]{enumitem}
\usepackage{balance}
\usepackage{adjustbox}
\usepackage{kantlipsum}
\usepackage{xspace}
\usepackage{mathrsfs}
\allowdisplaybreaks

\newtheorem{theorem}{Theorem}

\newtheorem{lemma}{Lemma}
\newtheorem{observation}{Observation}
\newtheorem{proposition}{Proposition}
\newtheorem{definition}{Definition}

\newtheorem{assumption}{Assumption}
\newtheorem{example}{Example}
\newtheorem{remark}{Remark}

\newcommand{\ourframework}{\texttt{\textsc{Fed-DPRoC}}\xspace}

\newcommand{\ourscheme}{\texttt{\textsc{RobAJoL}}\xspace}

\newcommand{\defeq}{\coloneqq}

\newcommand{\Totalclient}{\ensuremath{n}}
\newcommand{\Byzantine}{\ensuremath{b}}
\newcommand{\HonestSet}{\ensuremath{\mathcal{H}}}
\newcommand{\ByzSet}{\ensuremath{\mathcal{B}}}

\newcommand{\oringinD}{\ensuremath{d}}
\newcommand{\compressD}{\ensuremath{k}}

\newcommand{\JLepsilon}{\ensuremath{\epsilon_\mathrm{JL}}}
\newcommand{\JLmatrix}{\ensuremath{\bm{R}}}
\newcommand{\RobAvgCoeff}{\ensuremath{\kappa}}
\newcommand{\CountSketchBlock}{\ensuremath{s}}

\newcommand{\RobAvgError}{\ensuremath{\delta_\text{RA}}}

\newcommand{\DPepsilon}{\ensuremath{\epsilon_\mathrm{DP}}}
\newcommand{\DPdelta}{\ensuremath{\delta_\mathrm{{DP}}}}
\newcommand{\DPsigma}{\ensuremath{\sigma_\mathrm{{DP}}}}
\newcommand{\RDPalpha}{\ensuremath{\alpha}}
\newcommand{\RDPepsilon}{\ensuremath{\epsilon_\mathrm{RDP}}}
\newcommand{\ClipNorm}{\ensuremath{C}}
\newcommand{\NoiseMultiplier}{\sigma_\mathrm{{NM}}}

\newcommand{\Dataset}{\ensuremath{\mathcal{D}}}
\newcommand{\DatasetSize}{\ensuremath{\rho}}
\newcommand{\gradient}{\ensuremath{\bm{g}}}
\newcommand{\update}{\ensuremath{\bm{u}}}
\newcommand{\Loss}{\ensuremath{\mathcal{L}}}
\newcommand{\SampleLoss}{\ensuremath{\ell}}
\newcommand{\TotalGloablIter}{\ensuremath{T}}
\newcommand{\GlobalIter}{\ensuremath{t}}
\newcommand{\Minibatch}{\ensuremath{\tilde{\mathcal{D}}}}
\newcommand{\MinibatchSize}{\ensuremath{\varrho}}
\newcommand{\Momentum}{\ensuremath{\bm{m}}}
\newcommand{\MomentumCoeff}{\ensuremath{\beta}}
\newcommand{\LearningRate}{\ensuremath{\eta}}
\newcommand{\model}{\ensuremath{\bm{w}}}

\newcommand{\compMomentum}{\ensuremath{\hat{\Momentum}}}

\newcommand{\compressCall}{\textsc{Compress}\xspace}
\newcommand{\decompressCall}{\textsc{Decompress}\xspace}

\newcommand{\AGGCall}{\textsc{Agg}\xspace}

\def\BibTeX{{\rm B\kern-.05em{\sc i\kern-.025em b}\kern-.08em
    T\kern-.1667em\lower.7ex\hbox{E}\kern-.125emX}}

\begin{document}
\title{
% \ourframework: Communication-Efficient Differentially Private and Robust Federated Learning
Beyond Trade-offs: A Unified Framework for Privacy, Robustness, and Communication Efficiency in Federated Learning
}

\author{
{Yue Xia,~\IEEEmembership{Student Member,~IEEE}, Tayyebeh Jahani-Nezhad,~\IEEEmembership{Member,~IEEE}, Rawad Bitar,~\IEEEmembership{Member,~IEEE}}

\thanks{Y. Xia and R. Bitar are with the School of Computation, Information and Technology, Technical University of Munich, Germany (e-mails: \{yue1.xia, rawad.bitar\}@tum.de), T. Jahani-Nezhad is with the Electrical Engineering and Computer Science Department, Technische Universit\"at Berlin, Germany (e-mail: t.jahani.nezhad@tu-berlin.de). }
%% Many authors with many affiliations:
% \author{%
%     % \IEEEauthorblockN{Anonymous Authors}
%     \IEEEauthorblockN{Yue Xia\IEEEauthorrefmark{1},
%                     Tayyebeh Jahani-Nezhad\IEEEauthorrefmark{2}
%                     and Rawad Bitar\IEEEauthorrefmark{1}
%                     } % \vspace{-.5cm}
%   \IEEEauthorblockA{\IEEEauthorrefmark{1}%,
%                    School of Computation, Information and Technology,
%                    Technical University of Munich,
%                    Munich, Germany}
%   \IEEEauthorblockA{\IEEEauthorrefmark{2}%
%                     Department of Electrical Engineering and Computer Science, Technische Universität Berlin, Germany}
\thanks{Y. Xia and R. Bitar are supported by the German Research Foundation (DFG) grant agreement number BI 2492/1-1. T. Jahani-Nezhad was supported by the Gottfried Wilhelm Leibniz Prize 2021 of the German Research Foundation (DFG).}
\thanks{Preliminary results have been presented at IEEE FLTA25~\cite{xia2025fed}.}
 }

\maketitle

\begin{abstract}
We propose \ourframework, a novel federated learning framework designed to jointly provide differential privacy (DP), Byzantine robustness, and communication efficiency. Central to our approach is the concept of \textit{robust-compatible compression}, which allows reducing the bi-directional communication overhead without undermining the robustness of the aggregation. We instantiate our framework as \ourscheme, which integrates the Johnson-Lindenstrauss (JL)-based compression mechanism with robust averaging for robustness. Our theoretical analysis establishes the compatibility of JL transform with robust averaging, ensuring that \ourscheme maintains robustness guarantees, satisfies DP, and substantially reduces communication overhead. 
We further present simulation results on CIFAR-10, Fashion MNIST, and FEMNIST, validating our theoretical claims. We compare \ourscheme with a state-of-the-art communication-efficient and robust FL scheme augmented with DP for a fair comparison, demonstrating that \ourscheme outperforms existing methods in terms of robustness and utility under different Byzantine attacks.
% \yx{\\
% Page limit is 13 pages, including everything.  \\
% Journal website: https://www.comsoc.org/publications/journals/ieee-jsac/cfp/distributed-optimization-learning-and-inference-over \\
% Submission guideline: https://www.comsoc.org/publications/journals/ieee-journal-selected-areas-communication/submit-manuscript }
% \yx{Title options: \\
% 1. Unified Differential Privacy, Robustness, and Bidirectional Communication Efficiency for Federated Learning\\
% 2. Achieving Robust and Differentially Private Federated Learning with Bidirectional Communication Efficiency \\
% 3. On the Compatibility of Bidirectional Compression with Robust and Differentially Private Federated Learning\\
% 4. Bidirectional Robust-Compatible Compression for Differentially Private and Robust Federated Learning}
% \tj{I added a new title, feel free to revise it}
\end{abstract}
\begin{IEEEkeywords}
federated learning, Byzantine robustness, communication efficiency, differential privacy
\end{IEEEkeywords}

\section{Introduction}
% Federated learning (FL)~\cite{mcmahan2017communication}
% emerged as a method for training machine learning models over distributed, sensitive data held by multiple clients. 
% In FL, the federator (or server) maintains a global model that is iteratively refined: in each iteration, the federator broadcasts the current global model to the participating clients, who update the model using their local data, and return the updates to the federator. The federator aggregates them to update the global model. This iterative algorithm avoids transferring raw data to a central location, but introduces several fundamental challenges:
Federated learning (FL)~\cite{mcmahan2017communication} enables training machine learning models over distributed, sensitive data held by multiple clients. The federator (or server) iteratively broadcasts a global model, collects locally computed updates from clients, and aggregates them to refine the model. While raw data remain local, this iterative process introduces several fundamental challenges:
\begin{enumerate*}[label=(\emph{\alph*})]
    \item \textit{privacy}: local updates can leak information about the clients' data, e.g., through gradient inversion attacks~\cite{geiping2020inverting, zhu2019deep};
    \item \textit{robustness}: if care is not taken, malicious (Byzantine) clients may disrupt the training process %or cause the global model to become biased or fail to converge
    by manipulating their updates~\cite{blanchard2017machine}; and
    \item \textit{communication efficiency}: frequent exchange of high-dimensional model updates between the federator and the clients imposes a significant communication cost~\cite{kairouz2021advances,konevcny2016federated}.
\end{enumerate*} 
A practical FL framework should address all three challenges simultaneously~\cite{daly2024federated,guerraoui2024robust,konevcny2016federated,kairouz2021advances}. 

Since the federator may be honest-but-curious, privacy-preserving FL schemes can be split into two large categories depending on what the federator is allowed to see: 
\begin{enumerate*}[label=\emph{(\alph*)}]
    \item using secure aggregation, or
    \item using differential privacy (DP) mechanisms.
\end{enumerate*} 
Secure aggregation techniques ensure that the federator observes only an aggregation of the client updates, which is essential for the learning process to proceed~\cite{bonawitz2017practical,jahani2023swiftagg+,xia2024byzantine,hou2024priroagg}. DP-based approaches instead perturb the updates with noise~\cite{dwork2006differential} so that the semi-honest fedeartor observes either individual noisy updates, or a noisy version of the aggregated updates~\cite{chen2022fundamental}, i.e., using secure aggregation on top of DP.
Although secure aggregation avoids explicit noise insertion and thus does not incur any utility loss as with DP mechanisms, it creates high communication overhead and remains vulnerable to membership inference attacks~\cite{ngo2024secure}.

Robustness against malicious clients is typically achieved through robust aggregation rules such as Krum~\cite{blanchard2017machine}, coordinate-wise trimmed mean and median~\cite{yin2018byzantine}, Bulyan~\cite{guerraoui2018hidden}, and ProDiGy~\cite{ergisi2025prodigy}. These aggregation rules detect and downweight or remove outlier updates based on certain statistical or geometric metrics. Other defenses, such as~\cite{xhemrishi2025fedgt,cao2020fltrust}, assign scores to clients or their updates before aggregating them; however, they require the federator to have an auxiliary clean dataset representative of clients' data.

Communication efficiency is typically addressed through compression methods including quantization, sparsification, and distillation~\cite{konevcny2016federated, wen2017terngrad, bernstein2018signsgd, karimireddy2019error, reisizadeh2020fedpaq, rothchild2020fetchsgd, wu2022communication,egger2025bicompf}, which either reduce the dimension of the model updates or represent them using fewer bits.
% The former can be understood as mapping vectors from an alphabet $\mathcal{A}^\oringinD$ to an alphabet $\mathcal{A}^{\compressD}$ where $\compressD< d$, and the latter can be understood as a mapping $\mathcal{A}^d \to \mathcal{B}^d$ where $\mathcal{B}$ is an alphabet with fewer elements than $\mathcal{A}$.

Many works address pairs of these challenges. For instance, privacy and robustness are jointly considered in~\cite{so2020byzantine, jahani2023byzsecagg, xhemrishi2025fedgt, xia2024byzantine, xia2024lobyitfl, hou2024priroagg, guerraoui2021differential, ma2022differentially, allouah2023privacy}; communication efficiency and privacy are studied in~\cite{agarwal2018cpsgd, li2019privacy, liu2020fedsel, chen2022fundamental}; and communication efficiency and robustness are explored in~\cite{bernstein2018signsgd1, data2021byzantine, rammal2024communication, egger2025byzantine}.
Among these, \cite{rammal2024communication} proposed Byz-EF21 and Byz-EF21-BC, which utilize biased compressors, e.g., top-$k$ sparsification, to improve uplink communication efficiency (i.e., reducing the size of client-to-federator communication cost) and bidirectional communication efficiency (i.e., reducing both client-to-federator and federator-to-client communication cost), respectively. Their approach carefully structures model updates so that, when robust aggregation rules are applied, the algorithm still converges in the presence of malicious clients.

% However, overcoming all three challenges simultaneously remains underexplored. 
Despite these advances, addressing all three challenges simultaneously remains underexplored.
The approaches~\cite{li2022communication, zhu2022bridging, xiang2022distributed} use sign-based compressors to achieve both robustness and communication efficiency. These sign-based compressors are also differentially private, thereby offering a degree of privacy protection. Similarly, \cite{lang2023compressed} uses a lattice quantizer with randomized response, and~\cite{jin2024ternaryvote} applies a differentially private ternary-based compressor with majority voting. Although these methods unify the three objectives, they use aggressive compression mechanisms and lack proven robust aggregation rules, resulting in weaker robustness guarantees. 

In~\cite{zhang2023byzantine}, the authors employ top-$k$ sparsification, momentum-driven variance reduction, and DP. However, no convergence proof is provided, nor are the proven robust aggregation rules employed. In~\cite{hu2023efficient}, the robust aggregation rule FLTrust~\cite{cao2020fltrust} is integrated with compressive sensing. While the authors claim the scheme also preserves privacy, no formal privacy proof is given. Moreover, the approach assumes homogeneous client data, which is not realistic in FL scenarios.

In this paper, we propose \ourframework, a unified framework that simultaneously ensures provable DP and robustness guarantees, and offers flexible communication compression. The core idea is the introduction of \textit{robust-compatible compression}, i.e., compression methods that preserve the robustness guarantees of arbitrary robust aggregation rules. 

In each iteration of \ourframework, clients compute gradients based on the current global model, perturb them with DP noise, and apply a momentum-based update to strengthen robustness against malicious attacks~\cite{farhadkhani2022byzantine}. The clients then apply a robust-compatible compression mechanism and send the compressed updates to the federator. The federator aggregates the updates using any robust aggregation protocol and obtains a compressed aggregation. To further reduce the communication cost, % in this version of \ourframework, 
the federator delegates the decompression task to the clients. That is, the federator does not decompress the aggregation; instead, it broadcasts the compressed aggregation directly to the clients. The clients then decompress the aggregation, update the global model, and proceed to the next iteration.

To instantiate \ourframework, we propose \ourscheme, which employs Gaussian noise as a DP mechanism, and the Johnson-Lindenstrauss (JL) transform~\cite{johnson1984extensions} as a compression method. We theoretically prove that the JL transform satisfies the robust-compatible property under the robust averaging criterion given in~\cite{guerraoui2024robust}, i.e., the robustness guarantee that a robust aggregation holds in the original vector space is preserved in the compressed vector space resulting from the JL transform.
We further show that \ourscheme ensures formal DP guarantees and substantially reduces bidirectional communication overhead while maintaining robustness.

We provide extensive empirical evaluations validating the theoretical guarantees of \ourscheme across a range of compression rates, Byzantine attacks, and DP levels, using multiple benchmark datasets. 
% Our numerical results further demonstrate that \ourscheme outperforms variants that use top-$k$ sparsification (used in~\cite{zhang2023byzantine}), which is not robust-compatible with robust averaging\yx{(Prove it.)} \tj{Given that we now have a new comparison with top-k in Byz-EF21, should we still keep this previous result?}, when used instead of the JL transform. 
We compare \ourscheme with Byz-EF21-BC~\cite{rammal2024communication}, to which we additionally incorporate DP for fair comparison. Byz-EF21-BC makes the non-robust-compatible compressor top-$k$ sparsification ``compatible" with robust aggregation by modifying the clients’ model updates. Both schemes reduce bidirectional communication costs; however, our results show that \ourscheme achieves superior robustness and utility across a wide range of malicious attacks.

\section{System Model and Preliminaries}
\textbf{Notation}: For an integer $\Totalclient$, let $[\Totalclient]=\{1,\ldots,n\}$. Vectors and matrices are denoted by bold lowercase and uppercase letters, respectively. $(\mathbf{X})_{i,j}$ denotes the $(i,j)$-th entry of $\mathbf{X}$, and $\langle \bm{x},\bm{y}\rangle$ denotes the dot product of vectors $\bm{x}$ and $\bm{y}$. For $x_1,\ldots,x_n$ and $\mathcal{S}\subseteq[\Totalclient]$, let $x_{\mathcal{S}}=\{x_i:i\in\mathcal{S}\}$. Unless stated otherwise, $\| \cdot \|$ denotes the $\ell_2$ norm.
\subsection{Problem Formulation}
We consider an FL setting with $\Totalclient$ clients and a federator. Each client $i\in [\Totalclient]$ holds a local dataset $\Dataset_i$ consisting of $\DatasetSize_i$ data samples of the form $(\bm{x}_j,y_j) \in \mathcal{X}_i \times \mathcal{Y}_i \defeq \mathcal{D}_i$, where $\bm{x}$ is the input vector and $y$ is the corresponding label. We define $\Dataset \defeq \bigcup_{i=1}^n \Dataset_i$, $\mathcal{X} \defeq \bigcup_{i=1}^n \mathcal{X}_i$, and $\mathcal{Y} \defeq \bigcup_{i=1}^n \mathcal{Y}_i$.
The goal is to train a parameterized model $f: \mathcal{X} \times \mathbb{R}^\oringinD \to \mathcal{Y}$ (typically a neural network) defined by the parameter vector $\model \in \mathbb{R}^\oringinD$. This is done by optimizing the model parameter $\model$ such that for each input-label pair $(\bm{x}_j,y_j) \in \Dataset$, the model prediction $f(\bm{x}_j,\model)$ closely approximates the true label $y_j$.

The performance of the model is measured by a per-sample loss function $\SampleLoss: \mathbb{R}^\oringinD \times \mathcal{X} \times \mathcal{Y} \to \mathbb{R}_{+}$, often realized by mean squared error or cross-entropy. The local loss of client $i$ is %For each client $i$, we define the local loss as 

\begin{equation}
    \Loss(\model;\Dataset_i)\defeq \frac{1}{\DatasetSize_i}\sum\nolimits_{(\bm{x}_j,y_j) \in \Dataset_i}\SampleLoss(\model;\bm{x}_j,y_j).
\end{equation} 
The federator aims to address this optimization problem:
\begin{equation}
    \min_{\model} \Loss(\model) \defeq \min_{\model} \frac{1}{\Totalclient}\sum\nolimits_{i \in [\Totalclient]}\Loss(\model;\Dataset_i).
\end{equation}
The optimization is carried out using an iterative stochastic gradient descent algorithm, starting with an initial global model $\model^{(0)}$. %client $i \in [\Totalclient]$ also initializes its local models as $\model_i^{(0)}$. 
At each global iteration $\GlobalIter$, $0 \leq \GlobalIter < \TotalGloablIter$, the federator broadcasts the current global model $\model^{(\GlobalIter)}$ to all clients. Each client computes a local update as follows:
\begin{equation*}
\gradient_i^{(\GlobalIter)}=-\LearningRate_i^{(\GlobalIter)} \cdot \nabla \Loss(\model^{(\GlobalIter)}; \Minibatch_i^{(\GlobalIter)}),
\end{equation*}
where $\Minibatch_i^{(\GlobalIter)}$ is a minibatch of size $\MinibatchSize$, randomly sampled from $\Dataset_i$ via subsampling, and
% in each iteration with size $|\Minibatch_i^{(\GlobalIter)}|=\MinibatchSize$ through subsampling,
$\LearningRate_i^{(\GlobalIter)}$ is the local learning rate. Client $i$ sends the local update to the federator. The federator aggregates all received updates using an aggregation rule $\AGGCall$, i.e., $\update^{(\GlobalIter)}\!\!=\!\AGGCall(\gradient_1^{(\GlobalIter)}, \cdots, \gradient_\Totalclient^{(\GlobalIter)})$ and updates the global model $\model^{(\GlobalIter+1)}\!\!=\!\model^{(\GlobalIter)}\!-\LearningRate^{(\GlobalIter)} \!\cdot\! \update^{(\GlobalIter)}$, where $\LearningRate^{(\GlobalIter)}$ is the global learning rate.
% Consider $\LearningRate^{(\GlobalIter)}$ as the global learning rate, the global model is updated via 
% \begin{equation*}
%     \model^{(\GlobalIter+1)}=\model^{(\GlobalIter)}-\LearningRate^{(\GlobalIter)} \cdot \update^{(\GlobalIter)}.
% \end{equation*}

% \paragraph{Threat Model and Defense Goals}

In this setting, we assume that among $\Totalclient$ clients, $\Byzantine$ of them are malicious. Let $\ByzSet \subseteq \{ 1,\ldots,\Totalclient \}$ denote the set of malicious clients, with $|\ByzSet|=\Byzantine$, and $\HonestSet  =  [\Totalclient]\setminus \ByzSet$ % \subseteq \{ 1,\ldots,\Totalclient \}$ 
denote the set of honest clients, with $|\HonestSet|=\Totalclient-\Byzantine$. 
% We consider the federator to be honest-but-curious, i.e., it follows the protocol correctly but attempts to infer information about clients' data. 
We assume an honest-but-curious federator that follows the protocol but attempts to infer clients’ data.
Our goal is to design a federated learning scheme that simultaneously satisfies: 
\begin{enumerate*}[label=\emph{(\alph*)}]
    \item privacy, against both the federator and the clients;
    \item Byzantine robustness, against up to $\Byzantine$ malicious clients; and \item 
     communication efficiency, i.e., reducing the bidirectional communication cost.
\end{enumerate*}

\begin{figure*}[!tb]
    \centering
    % \hspace{-7.5cm}
    \begin{adjustbox}{center}
        \resizebox{0.9\textwidth}{!}{\input{figure/framework}}
    \end{adjustbox}
    \caption{Illustration of \ourframework. 
    Clients performs: $(1)$ minibatch sampling; $(2)$ decompression and model update; $(3)$ gradient computation; $(4)$ DP enhancement (\cref{eq:dp_enforce}); $(5)$ momentum computation (\cref{eq:momentum}); $(6)$ compression; and $(7)$ communication. Honest clients send their true compressed momenta to the federator, while malicious clients send arbitrarily manipulated vectors. The honest-but-curious federator: $(8)$ applies robust aggregation; and $(9)$ broadcasts the compressed aggregation.
    % Clients perform the following: $(1)$ sample a minibatch; $(2)$ decompress the compressed global model received from the federator in the previous iteration using the \decompressCall and update the model; $(3)$
    % compute the gradient based on the current global model; $(4)$ perform the DP mechanism \cref{eq:dp_enforce}; $(5)$ compute momentum updates \cref{eq:momentum}; $(6)$ apply a robust-compatible compression on their momenta; and $(7)$ send their compressed momenta to the federator, who is honest-but-curious. Honest clients send their true compressed momenta, while malicious clients send arbitrarily manipulated vectors. The federator: $(8)$ applies robust aggregation \AGGCall on the compressed momenta to mitigate the impact of malicious clients; and $(9)$ broadcasts the compressed aggregation.
    }  
    \label{fig:framework}
    \vspace{-0.08in}
\end{figure*}

\subsection{Preliminaries}
\subsubsection{Robustness (Convergence)}
% Consider in the federated learning setting with $\Totalclient$ clients, $\Byzantine$ of them are Byzantine and may behave maliciously. 
In the presence of malicious clients, minimizing the global average loss 
$\frac{1}{\Totalclient}\sum_{i=1}^{\Totalclient} \Loss(\model;\Dataset_i)$ is generally not meaningful, as malicious clients can arbitrarily distort the optimization objective. Therefore, the goal changes to minimizing the average loss over the set of honest clients~\cite{guerraoui2024robust} defined as
$
\Loss_\HonestSet(\model) \coloneqq \frac{1}{|\HonestSet|}\sum_{i\in \HonestSet}\Loss(\model;D_i).
$
\begin{definition}[$(b,\nu)$-robust algorithm~\cite{guerraoui2024robust}] \label{def:robustness}
    A distributed algorithm is called $(b,\nu)$-robust if it outputs a model $\hat{\model}$ such that 
    % \begin{align*}
        $\mathbb{E}[\Loss_\HonestSet(\hat{\model})-\Loss^{\star}] \leq \nu,$
    % \end{align*}
    where $\Loss^{\star} \coloneqq \inf_{\model \in \mathbb{R}^d} \Loss_\HonestSet(\model)$ is the optimal loss over the honest clients. 
\end{definition}

\subsubsection{Differential Privacy} We consider the following.
\begin{definition}[$(\DPepsilon,\DPdelta)$-Differential Privacy (DP)~\cite{dwork2006our}]
\label{def:dp}
    Let $\DPepsilon \geq 0, 0\leq \DPdelta \leq 1$. A randomized mechanism $f: \mathcal{X}^\DatasetSize \mapsto \mathcal{Y}$ satisfies $(\DPepsilon,\DPdelta)$-DP if, for any pair of adjacent datasets $\Dataset, \Dataset' \in \mathcal{X}^\DatasetSize$ and any measurable set $\mathcal{S} \subseteq \mathcal{Y}$, it holds that
    \begin{equation*}
        \Pr [f(\Dataset)\in \mathcal{S}] \leq \mathrm{e}^{\DPepsilon} \cdot \Pr [f(\Dataset')\in \mathcal{S}] + \DPdelta,
    \end{equation*}
    where $\Dataset$ and $\Dataset'$ differ in at most one element.
    % can be obtained from each other by adding or removing at most one element.
\end{definition}
Standard DP provides strong individual-level guarantees but can lead to loose bounds under composition. To address this, we use Rényi Differential Privacy (RDP), which offers a tighter analysis of cumulative privacy loss.

\begin{definition}[Rényi Differential Privacy (RDP)~\cite{mironov2017renyi}]
    Let $\RDPalpha>1$ and $\RDPepsilon>0$. A randomized mechanism $f: \mathcal{X}^\DatasetSize \mapsto \mathcal{Y}$ satisfies $(\RDPalpha,\RDPepsilon)$-RDP if, for any pair of adjacent datasets $\Dataset, \Dataset' \in \mathcal{X}^\DatasetSize$,  the Rényi divergence of order $\RDPalpha$ between the output distributions satisfies
    % \begin{equation*}
     $  D_{\RDPalpha} \left( f(\Dataset)\Vert f(\Dataset') \right) \leq \RDPepsilon$,
    % \end{equation*}
    where $D_{\RDPalpha} \!\! \left( f(\Dataset)\Vert f(\Dataset') \right) \! \defeq \!\frac{1}{\RDPalpha-1}\log \mathbb{E}_{z\sim f(\Dataset')} \!\! \left[ \left( \frac{f(\Dataset)(z)}{f(\Dataset')(z)} \right) ^\RDPalpha\right]$.
\end{definition}
\begin{proposition}[RDP to DP conversion~\cite{mironov2017renyi}]
    If a mechanism $f$ satisfies $(\RDPalpha,\RDPepsilon)$-RDP, then it also satisfies $(\DPepsilon,\DPdelta)$-DP for any $0 < \DPdelta < 1$, where $\DPepsilon=\RDPepsilon+\frac{\log (1/\DPdelta)}{\RDPalpha-1}$.
\end{proposition}

\section{The Proposed Framework: \ourframework}
We present \ourframework, a unified FL framework that jointly achieves privacy, Byzantine robustness, and communication efficiency. \ourframework protects individual data through DP and significantly reduces the communication cost between clients and the federator by leveraging gradient compression techniques—while maintaining model utility and robustness.
We illustrate \ourframework in \cref{fig:framework}, and describe it in detail below. The algorithm is given in \cref{algo:our_framework}.

% \begin{figure}[!tb]
%     \centering
%     \includegraphics[width=0.95\linewidth, height=0.13\textheight]{figure/Picture11.png}
%     \caption{\ourframework}
%     \label{fig:framework}
% \end{figure}

In each global iteration $\GlobalIter$, $1\leq \GlobalIter < \TotalGloablIter$, the \emph{clients} performs the following steps based on the current global model:
% the federator broadcasts the current global model $\model^{(\GlobalIter)}$ to all clients.  Upon receiving it, each client $i \in [\Totalclient]$ ($i \in \HonestSet$) performs the following steps:
\begin{enumerate}
    \item[\textbullet] \textbf{Minibatch Sampling and Gradients Computation}. Client $i$ draws a minibatch $\Minibatch_i^{(\GlobalIter)}$ of size $\MinibatchSize$ from its local dataset. Different sampling techniques (e.g., with/without replacement or Poisson sampling) are possible and affect the privacy guarantees. Client $i$ then computes per-sample gradients $\nabla \! \SampleLoss{ \left(  \model^{(\GlobalIter)}; \bm{x}_j,y_j  \right) }$ for all $(\bm{x}_j,y_j) \in  \Minibatch_i^{(\GlobalIter)}.$
    % \item[\textbullet] 
    % \textbf{Gradients Computation}. Compute per-sample gradients $\nabla \SampleLoss{\left( \model^{(\GlobalIter)}; \bm{x}_j,y_j \right)}$ for all $(\bm{x}_j,y_j) \in \Minibatch_i^{(\GlobalIter)}$.
    \item[\textbullet] \textbf{Differential Privacy Enforcement}. Clip the per-sample gradients, average them, and add DP noise, i.e., compute
    \begin{align}\label{eq:dp_enforce}
        \gradient_i^{(\GlobalIter)} \defeq \frac{1}{\MinibatchSize} \hspace{-1mm}\sum_{(\bm{x}_j,y_j) \in \Minibatch_i^{(t)}}\hspace{-4mm} \text{Clip}(\nabla \SampleLoss(\model^{(\GlobalIter)}; \bm{x}_j, y_j);\ClipNorm)+ \bm{n}_{\text{DP}},
    \end{align}
    % where $I_\oringinD$ is the identity matrix of dimension $\oringinD \times \oringinD$ and $\DPsigma$ is the noise scale. Additionally, 
    where $\text{Clip}(\gradient; \ClipNorm) \!\!\coloneqq\!\! \gradient \! \cdot \! \min \!\left\{ \!1, \! \frac{\ClipNorm}{\gradient} \! \right\} \!$, $\ClipNorm$ is the clipping norm for gradient vector, and $\bm{n}_{\text{DP}}$ is a noise vector drawn according to a predetermined distribution, such as the Laplace or the Gaussian distribution. The choice of the distribution affects the privacy guarantees, e.g.,~\cite{dwork2006calibrating,dwork2014algorithmic}. 
    \item[\textbullet] \textbf{Momentum Computation}. For $\GlobalIter \geq 1$ and $\Momentum_{i}^{(0)}\!\defeq\!\bm{0}$, update momentum using the noisy gradient~\cite{polyak1964some} as
    \begin{align}\label{eq:momentum}
         \Momentum_{i}^{(\GlobalIter)}\defeq \MomentumCoeff^{(\GlobalIter-1)} \Momentum_{i}^{(\GlobalIter-1)} + (1-\MomentumCoeff^{(\GlobalIter-1)}) \gradient_i^{(\GlobalIter)},
    \end{align}
    where $0\!\leq\! \MomentumCoeff^{(\GlobalIter-1)} \!\leq\!1$ is the momentum coefficient.
    \item[\textbullet] \textbf{Compression}. Apply a \emph{robust-compatible compression} method \compressCall to reduce the uplink communication cost, resulting in $\compMomentum_{i}^{(\GlobalIter)} \coloneqq$ \compressCall $(\Momentum_{i}^{(\GlobalIter)})$.
     \item[\textbullet] \textbf{Communication}. Send $\compMomentum_{i}^{(\GlobalIter)}$ to the federator.
\end{enumerate}
% }

\begin{figure}[!tb]
    \centering
    \hspace{-1.5cm}
    \begin{adjustbox}{center}
        \resizebox{0.45\textwidth}{!}{\input{figure/compression}}
    \end{adjustbox}
    \caption{A robust-compatible compression ensures that if $\AGGCall$ is a robust aggregation rule, i.e., $\Momentum_\AGGCall$ is faithful to the average of the benign vectors $\Momentum_1, \cdots, \Momentum_\Totalclient$, the composition $\compressCall \circ \AGGCall \circ\decompressCall$ also gives an output %\decompressCall($\compMomentum_\AGGCall$)
    that is faithful to the average of those benign vectors.}
    \label{fig:robust_agg_comp}
    \vspace{-0.08in}
\end{figure}

\noindent After receiving the compressed momenta from the clients, the \emph{federator} performs the following steps:
\begin{enumerate}
    \item[\textbullet] \textbf{Robust aggregation}. Use a robust aggregation rule \AGGCall to aggregate the compressed momenta, i.e., compute 
    \begin{equation}\label{eq:robust_agg}\compMomentum_\AGGCall^{(\GlobalIter)} = \Call{Agg}{\compMomentum_{1}^{(\GlobalIter)}, \cdots, \compMomentum_{\Totalclient}^{(\GlobalIter)}}.
    \end{equation}
    \item[\textbullet] \textbf{Comunication}. Broadcast $\compMomentum_\AGGCall^{(\GlobalIter)}$ to the clients.
\end{enumerate}

\noindent Upon receiving the compressed aggregation from the federator, the \emph{clients} perform:
\begin{enumerate}
    \item[\textbullet] \textbf{Decompression}. Apply the corresponding decompression on $\compMomentum_\AGGCall^{(\GlobalIter)}$, i.e., $\update^{(\GlobalIter)}= \decompressCall(\compMomentum_\AGGCall^{(\GlobalIter)})$.
    \item[\textbullet] \textbf{Update.} Compute the global model for the next iteration $\model^{(\GlobalIter+1)} = \model^{(\GlobalIter)} - \LearningRate^{(\GlobalIter)} \bm{u}^{(\GlobalIter)}$. 
    \item[\textbullet] \textbf{Next Iteration}.
\end{enumerate}

\ourframework introduces the concept of \emph{robust-compatible compression} (\cref{fig:robust_agg_comp}) that allows us to integrate compression into robust aggregation while preserving robustness.

\begin{definition}[Robust-compatible Compression]\label{def:robust_comp_comp}
    Let $\AGGCall: (\mathbb{R}^\oringinD)^\Totalclient\to\mathbb{R}^\oringinD$ be a robust aggregation rule that takes as input $\Totalclient$ vectors in $\mathbb{R}^\oringinD$, out of which $\Byzantine<\Totalclient/2$ can be arbitrarily manipulated, and outputs an aggregation in $\mathbb{R}^\oringinD$ that is faithful to the average of the $\Totalclient-\Byzantine$ benign vectors; where the faithfulness is measured according to a desired robustness criterion.
    Let $\compressCall:\mathbb{R}^d \to \mathbb{R}^k$, with $k\ll d$, be a compression mechanism that takes a vector in $\mathbb{R}^d$ and compresses it to a vector in $\mathbb{R}^k$, and let $\decompressCall:\mathbb{R}^k \to \mathbb{R}^d$ be the corresponding decompression rule.
    
    \compressCall is said to be \emph{robust-compatible} according to the desired robustness criterion if the following holds: the composition $\compressCall \circ \AGGCall \circ\decompressCall$ remains robust according to the desired robustness criterion. In other words, when $\AGGCall$ is applied on the compressed version of the $n$ vectors (now in $\mathbb{R}^k$), it outputs an average vector (also in $\mathbb{R}^k$) that, after decompression to $\mathbb{R}^d$, remains faithful to the average of the $n-b$ benign original vectors.
\end{definition}

In FL, such compatibility ensures that the system remains robust to adversarial behavior even after dimensionality reduction, while significantly reducing communication costs.

%----------------------------------------------------------
%%%%%% Algorithmic description of the framework%%%%%%%%%%%
%----------------------------------------------------------
\begin{algorithm}[tb!]
\footnotesize
\caption{\ourframework}
\label{algo:our_framework}
\textbf{Initialization:} 
Initialize the model $\model^{(0)}$; set initial momentum $\Momentum_i^{(0)} = \bm{0}$; define mini-batch size $\MinibatchSize$, learning rates $\LearningRate$, robust aggregation rule $\AGGCall$, total number of global iterations $\TotalGloablIter$.

\begin{algorithmic}[1]
\For{$t=1, \cdots, T-1$}

\For{each honest client $i \in \HonestSet$ in parallel}
\State Sample a mini-batch $\Minibatch_i^{(\GlobalIter)}$ of size $\MinibatchSize$ from $\Dataset_i$.
\State 
\vspace{-0.15in}
    \begin{align*}
        \gradient_i^{(\GlobalIter)} \defeq \frac{1}{\MinibatchSize} \hspace{-1mm}\sum\nolimits_{(\bm{x}_j,y_j) \in \Minibatch_i^{(t)}}\hspace{-1mm} \text{Clip}(\nabla \SampleLoss(\model^{(\GlobalIter)}; \bm{x}_j, y_j);\ClipNorm)+ \bm{n}_{\text{DP}}
    \end{align*}
    \vspace{-0.1in}
%
% \State Apply gradient clipping.
% \State Add noise to ensure DP.
% $\gradient_i^{(\GlobalIter)}=$\Call{DP}{$\model^{(\GlobalIter)}, \Minibatch_i^{(\GlobalIter)}, \MinibatchSize, \DPsigma, \ClipNorm$}.
\State 
    $\Momentum_{i}^{(\GlobalIter)}\defeq \MomentumCoeff^{(\GlobalIter-1)} \Momentum_{i}^{(\GlobalIter-1)} + (1-\MomentumCoeff^{(\GlobalIter-1)}) \gradient_i^{(\GlobalIter)}$

\State $\compMomentum_{i}^{(\GlobalIter)} \coloneqq$ \compressCall $(\Momentum_{i}^{(\GlobalIter)})$
% \State Compress the momentum to $\compMomentum_{i}^{(\GlobalIter)}$.
\State Send $\compMomentum_i^{(t)}$ to the federator.
\EndFor

\For{each malicious client $i \in \ByzSet$ in parallel}
\State Generate an arbitrary vector $\compMomentum_i^{(t)} \in \mathbb{R}^{\compressD}$.
\State Send $\compMomentum_i^{(t)}$ to the federator.
\EndFor

% \State At the federator do:
\State $\compMomentum_\AGGCall^{(\GlobalIter)} = \Call{Agg}{\compMomentum_{1}^{(\GlobalIter)}, \cdots, \compMomentum_{\Totalclient}^{(\GlobalIter)}}$
\State Federator broadcasts $\compMomentum_\AGGCall^{(\GlobalIter)}$ to all clients.

\For{each honest client $i \in \HonestSet$ in parallel}
\State $\update^{(\GlobalIter)}=\decompressCall(\compMomentum_\AGGCall^{(\GlobalIter)})$
\State $\model^{(\GlobalIter+1)}=\model^{(\GlobalIter)}-\LearningRate \cdot \update^{(\GlobalIter)}$

% \If{$t=T-1$}
%     \State Send $\model^{(\GlobalIter+1)}$ to the federator.
% \EndIf

\EndFor

\EndFor
\State \textbf{return} $\tilde{\model}$ uniformly sampled from $\{ \model^{(0)}, \cdots,  \model^{(\TotalGloablIter)} \}$.
\end{algorithmic}
\end{algorithm}
%----------------------------------------------------------
%%%%%% END Algorithmic description of the framework  %%%%%%
%----------------------------------------------------------

\section{\ourscheme: an Instantiation of \ourframework}
In this section, we present \ourscheme, one instantiation of \ourframework with a specific noise mechanism, robust aggregation, and compression scheme. The key idea of \ourscheme lies in its provably robust, private, and communication-efficient design, achieved by integrating the Johnson–Lindenstrauss (JL) transform~\cite{johnson1984extensions} as the compression method and a $(\Byzantine,\RobAvgCoeff)$-robust averaging~\cite{guerraoui2024robust} to ensure Byzantine robustness. This instantiation demonstrates how \ourframework can be realized using principled components with theoretical guarantees. 

\subsection{Johnson–Lindenstrauss (JL) Compression}
To reduce communication overhead, \ourscheme employs the JL transform as the compression method \compressCall. JL transform is the compression method satisfying JL lemma with high probability.
The JL lemma guarantees that high-dimensional vectors can be projected into a lower-dimensional subspace with only a small distortion in pairwise distances.

\begin{definition}[JL Lemma~\cite{johnson1984extensions,dasgupta2003elementary}] \label{def:JL_lemma}
For a given real number $\JLepsilon\in(0,1)$, for every set $\mathcal{V}\subset\mathbb{R}^{d}$, let there be a positive integer $\compressD$ such that $\compressD=\mathcal{O}(\JLepsilon^{-2}\log{|\mathcal{V}|})$, there exists
a Lipschitz mapping ${f}:\mathbb{R}^\oringinD \rightarrow \mathbb{R}^\compressD$, such that for all $\bm{u,v}\in\mathcal{V}$, we have
\begin{align*}
    (1-\JLepsilon)\|\bm{u}-\bm{v}\|^2 \leq \| f(\bm{u})-f(\bm{v}) \|^2 \leq (1+\JLepsilon)\|\bm{u}-\bm{v}\|^2.
\end{align*}
\end{definition}
The JL Lemma states that any subset $\mathcal{V}\subset \mathbb{R}^d$ can be embedded into $\mathbb{R}^{\compressD}$, with $\compressD\ll\oringinD$, while preserving pairwise Euclidean distances within a distortion of $\JLepsilon$. Importantly, for a given $\JLepsilon$, the required dimension $\compressD$ is independent of $d$ and scales logarithmically with $|\mathcal{V}|$. Among all variants of JL transforms~\cite{indyk1998approximate,frankl1988johnson,achlioptas2003database,ailon2006approximate,kane2014sparser}, we provide the following example.

% Below, we provide examples of a JL transform.

% \begin{example}[Gaussian JL Matrix~\cite{indyk1998approximate}]
% \label{def:gaussian_jl}
%     For a given $\JLepsilon \in (0, 1)$, a Gaussian JL matrix is a random matrix $\JLmatrix \in \mathbb{R}^{\compressD \times d}$, whose entries are drawn i.i.d. from the normal distribution $\mathcal{N}(0, \frac{1}{\compressD})$, with $\compressD = \mathcal{O}\left(\JLepsilon^{-2}{\log n}\right)$.
% \end{example}

% \cref{def:gaussian_jl} gives an easy way to generate a Lipschitz mapping satisfying the JL lemma with high probability, however, it comes with a high computational cost. To improve computational efficiency, several variants of generating random linear maps using, e.g., sparser mapping matrix, have been proposed~\cite{frankl1988johnson, achlioptas2003database,ailon2006approximate,kane2014sparser}. 

\begin{example}
    [Count-Sketch JL Transform \cite{chen2022fundamental, kane2014sparser}]  
    \label{def:sparser_jl}
Consider a set $\mathcal{V}\subset \mathbb{R}^\oringinD$. Let $\compressD, \CountSketchBlock \in \mathbb{N}$ such that $\CountSketchBlock$ divides $\compressD$, i.e., $\compressD = \CountSketchBlock \cdot p$ for some integer $p$. A matrix $\JLmatrix \in \mathbb{R}^{\compressD \times \oringinD}$ is constructed as follows to serve as a JL transform:

\begin{itemize}
    \item Construct $p$ independent block matrices $\JLmatrix_1, \dots, \JLmatrix_p \in \{-1, 0, 1\}^{s \times d}$, where each entry of $\JLmatrix_i$ is generated using
    \[
     (\JLmatrix_i)_{j,l}=\zeta_i(l) \cdot \mathbbm{1}_{ \{h_i(l)=j \}}
    \]
    for $j \in [s]$, $l \in [d]$. Here,
    $h_i: [\oringinD] \to [\CountSketchBlock]$ and $\zeta_i: [\oringinD] \to \{-1,  +1\}$ are independent hash functions.
    \item Stack the $p$ blocks vertically and scale to construct $\JLmatrix$:
    \[
        \JLmatrix = \frac{1}{\sqrt{p}}[\JLmatrix_1^T, \cdots, \JLmatrix_p^T]^T.
    \]
\end{itemize}
This construction yields a sparse random projection matrix that satisfies the JL lemma with high probability, i.e., for a set $\mathcal{V} \subset \mathbb{R}^{\oringinD}$ of cardinality $|\mathcal{V}|$, the random projection matrix satisfies the JL lemma with probability $1-\frac{1}{|\mathcal{V}|}$~\cite{kane2014sparser}. %Specifically, if $\JLmatrix$ is constructed as above and $\compressD = \mathcal{O}(\JLepsilon^{-2} \log \xi)$, then for a set $\mathcal{X} \subset \mathbb{R}^{\oringinD}$ of size $|\mathcal{X}|=\xi$, all pairwise Euclidean distances between points in $\mathcal{X}$ are preserved up to distortion $\JLepsilon \in (0, 1)$.
\end{example}

In \ourscheme, the compression method is implemented as  $\compMomentum_{i}=\compressCall(\Momentum_{i}^{(\GlobalIter)})\defeq\JLmatrix \Momentum_{i}^{(\GlobalIter)}$, where $\JLmatrix$ is a random JL transform matrix generated as in \cref{def:sparser_jl}. We denote the compression rate by $\frac{\oringinD}{\compressD}$. The decompression is defined as $\decompressCall(\compMomentum_\AGGCall)\defeq\JLmatrix^T \compMomentum_\AGGCall$. Since $\JLmatrix^T\JLmatrix$ is not necessarily equal to the identity matrix, but behaves like it in expectation, this decompression approximates the original result. The expected decompression error can be bounded as in \cref{lem:jl_property}, which states useful properties of the JL transform.

\begin{lemma}
\label{lem:jl_property}
Consider a set $\mathcal{V} \!\subset\! \mathbb{R}^d$ and  $k\!=\!\mathcal{O}(\JLepsilon^{-2}\log|\mathcal{V}|) \!\ll \! d$. A JL transform matrix $\JLmatrix \in \mathbb{R}^{k\times d}$ constructed randomly satisfies the following properties for all vectors $\bm{v}\in \mathbb{R}^d$:
\begin{enumerate}
    \item $\| \JLmatrix\bm{v} \|^2$ is an unbiased estimator of $\| \bm{v} \|^2$~\cite{vempala2005random,achlioptas2003database, kane2014sparser}, i.e.,
    % \begin{align*}
        $\mathbb{E}_{\JLmatrix}[\| \JLmatrix\bm{v} \|^2]=\| \bm{v} \|^2;$
    % \end{align*}
    \item the expected decompression error $\mathbb{E}_{\JLmatrix}[\| \JLmatrix^T\JLmatrix\bm{v} -\bm{v}\|^2]$ is bounded as~\cite{song2021iterative, chen2022fundamental}, 
    % \begin{align*}
        $\mathbb{E}_{\JLmatrix}[\| \JLmatrix^T\JLmatrix\bm{v} -\bm{v}\|^2] \leq \frac{\tau\oringinD}{\compressD} \| \bm{v} \|^2,$
    % \end{align*}
    where $\tau$ is a small constant, typically in the range  $\tau\in[1,3]$ depending on the distribution of $\JLmatrix$; and
    \item the decompression error satisfies 
    \begin{align*}
        \| \JLmatrix^T\JLmatrix\bm{v} -\bm{v}\|^2 \leq \JLepsilon^2 \| \bm{v} \|^2
    \end{align*}
    with probability $1-\frac{1}{|\mathcal{V}|}$. Here, the probability is over the random choice of the JL transform matrix $\JLmatrix$. However, for a given matrix that is known to satisfy the JL lemma, i.e., \cref{def:JL_lemma}, the above statement is deterministic.
    \item  With probability $1-\frac{1}{|\mathcal{V}|}$, the matrix $\JLmatrix$ satisfies
    \begin{align*}
        \| \JLmatrix \|^2 = \| \JLmatrix^T \|^2 \leq 1+\JLepsilon.
    \end{align*}
   Here, the probability is also over the random choice of the JL transform matrix $\JLmatrix$ as explained above.
    % However, for a given matrix that is known to satisfy the JL lemma, i.e., \cref{def:JL_lemma}, the above statement is deterministic.
    % for ensuring that the random matrix $\JLmatrix$ is a JL transform matrix. However, for a given JL transform matrix, the above statement is deterministic.
\end{enumerate}
\end{lemma}

The proof of properties $\textit{3)}$ and $\textit{4)}$ is given in Appendix~\ref{proof:lemma1}.

\subsection{Robust Averaging and Compatibility of JL}
To ensure Byzantine robustness, we employ an aggregation rule \AGGCall that satisfies $(\Byzantine,\RobAvgCoeff)$-Robust Averaging~\cite{guerraoui2024robust}, i.e., \cref{def:rob_avg}. We extend the definition of $(\Byzantine,\RobAvgCoeff)$-robust averaging to $(\Byzantine,\RobAvgCoeff, \RobAvgError)$-robust averaging, i.e., \cref{def:rob_avg_new}, with $0 \leq \RobAvgError \leq 1$. We prove that, with a constraint on the range of $\RobAvgError$, $(\Byzantine,\RobAvgCoeff, \RobAvgError)$-robust averaging provides convergence guarantees. We also prove that given a robust aggregation which is $(\Byzantine,\RobAvgCoeff)$-robust averaging, the composition $\compressCall \circ \AGGCall \circ\decompressCall$, with JL transform as the compressor, is $(\Byzantine,\RobAvgCoeff', \RobAvgError)$-robust averaging.
% is a criterion characterizing the ability of the robust aggregation rule to protect against corrupt updates.

\begin{definition}[$(\Byzantine,\RobAvgCoeff)$-Robust Averaging~\cite{guerraoui2024robust}]
\label{def:rob_avg}
    For a non-negative integer $\Byzantine<\Totalclient/2$, an aggregation rule $\AGGCall$: $(\mathbb{R}^\oringinD)^{\Totalclient} \rightarrow \mathbb{R}^\oringinD $ is a $(\Byzantine,\RobAvgCoeff)$-robust averaging aggregation if there exists a (non-negative) real valued $\RobAvgCoeff$ such that, for any set of $\Totalclient$ vectors $\bm{v}_1, \cdots, \bm{v}_\Totalclient \in \mathbb{R}^\oringinD$ and any subset $S \in [\Totalclient]$ of size $\Totalclient-\Byzantine$, the following holds:
    % \begin{align*}
        $\| \AGGCall(\bm{v}_1, \cdots, \bm{v}_\Totalclient)-\bar{\bm{v}}_S \|^2 \leq \frac{\RobAvgCoeff}{|S|} \sum_{i\in S} \| \bm{v}_i-\bar{\bm{v}}_S \|^2\!,$
    % \end{align*}
    where $\bar{\bm{v}}_S \defeq \frac{1}{|S|}\sum_{i\in S} \bm{v}_i$. The parameter $\RobAvgCoeff$ is referred to as the robustness coefficient.
\end{definition}

We extend the definition of $(\Byzantine,\RobAvgCoeff)$-robust averaging by allowing an extra error term.

\begin{definition}[$(\Byzantine,\RobAvgCoeff,\RobAvgError)$-Robust Averaging]
\label{def:rob_avg_new}
    For a non-negative integer $\Byzantine<\Totalclient/2$, an aggregation $\AGGCall$: $(\mathbb{R}^\oringinD)^{\Totalclient} \rightarrow \mathbb{R}^\oringinD $ is $(\Byzantine,\RobAvgCoeff,\RobAvgError)$-robust averaging if there exists a (non-negative) real valued $\RobAvgCoeff$ and  a real valued $0 \leq \RobAvgError \leq 1$ such that, for any set of $\Totalclient$ vectors $\bm{v}_1, \cdots, \bm{v}_\Totalclient \in \mathbb{R}^\oringinD$ and any subset $S \in [\Totalclient]$ of size $\Totalclient-\Byzantine$, let $\bar{\bm{v}}_S \defeq \frac{1}{|S|}\sum_{i\in S} \bm{v}_i$, the following holds:
    \begin{align*}
        \| \AGGCall(\bm{v}_1, \cdots, \bm{v}_\Totalclient)\!-\!\bar{\bm{v}}_S \|^2 \!\leq\! \frac{\RobAvgCoeff}{|S|} \sum_{i\in S} \| \bm{v}_i-\bar{\bm{v}}_S \|^2 \!+\! \RobAvgError\! \cdot \! \| \bar{\bm{v}}_S \|^2 \!.
    \end{align*}
    % The parameter $\RobAvgCoeff$ is referred to as the robustness coefficient.
\end{definition}

Notice that for $\RobAvgError=0$, $(\Byzantine,\RobAvgCoeff, \RobAvgError)$-robust averaging reduces to $(\Byzantine,\RobAvgCoeff)$-robust averaging.

% \emph{Compressed robust aggregation:}
In \ourscheme, the robust aggregation rule works as follows. Client $i$'s momentum update $\Momentum_i \in \mathbb{R}^\oringinD$ (we omit the notation of global iteration $\GlobalIter$ for brevity) is compressed to $\compMomentum_i \in \mathbb{R}^\compressD$ using a JL transform random matrix $\JLmatrix$. The compressed vectors are aggregated using a $(\Byzantine,\RobAvgCoeff)$-robust averaging rule $\AGGCall$ to obtain $\compMomentum_\AGGCall \in \mathbb{R}^k$. Then, $\decompressCall(\compMomentum_\AGGCall) =\JLmatrix^T \compMomentum_\AGGCall$ is computed. 
In \cref{pro:JLwithRobAvg_deterministic}, we show that given a matrix that satisfies the JL lemma, see \cref{def:JL_lemma}, and an aggregation rule that is $(\Byzantine,\RobAvgCoeff)$-robust-averaging, see \cref{def:rob_avg}, the composition $\compressCall\circ\AGGCall\circ\decompressCall$ satisfies the $(\Byzantine,\RobAvgCoeff', \RobAvgError)$-robust averaging criterion, where $\RobAvgCoeff'$ is close to $\RobAvgCoeff$ and $\RobAvgError$ is small, hence implying that the JL transform is robust-compatible under the robust averaging criterion, according to \cref{def:robust_comp_comp}.

\begin{proposition}
    \label{pro:JLwithRobAvg_deterministic}
    % Any matrix $\JLmatrix$ that satisfies the JL lemma, which we call a JL matrix, yields a robust-compatible compression under the robust averaging criterion. 
    Given a matrix $\JLmatrix$ that satisfies the JL lemma, which we call a JL matrix, with distortion parameter $\JLepsilon$, and a $(\Byzantine,\RobAvgCoeff)$-robust averaging rule $\AGGCall$, the composition $\compressCall\circ\AGGCall\circ\decompressCall$ with the JL matrix as the compressor is $(\Byzantine,\RobAvgCoeff', \RobAvgError)$-robust averaging, with $\RobAvgCoeff'=(1+\JLepsilon)^2\cdot \RobAvgCoeff \text{\;\; and \;\;} \RobAvgError=\JLepsilon^2.$
    In other words, with small $\JLepsilon$, any JL matrix yields a robust-compatible compression under the robust averaging criterion. 
\end{proposition}

\begin{proof}
    Let $\AGGCall$ be a robust aggregation rule that satisfies $(\Byzantine,\RobAvgCoeff)$-robust averaging. We consider the case where $\AGGCall$ is combined with a JL matrix as the compression. 
    
    Given momenta vectors $\Momentum_1, \cdots, \Momentum_\Totalclient$. Let $\JLmatrix$ be a JL matrix, i.e., a linear map satisfying \cref{def:JL_lemma}. Let  $\compMomentum_i=\JLmatrix\Momentum_i$, for $i \in [\Totalclient]$, denote the compressed inputs, and let $\compMomentum_{\AGGCall}=\AGGCall(\compMomentum_1, \cdots, \compMomentum_\Totalclient)$ be the aggregation of $\AGGCall$. Following the definition of $(\Byzantine,\RobAvgCoeff)$-robust averaging, we have
    \begin{align}\label{eq:comp_kappa_averaging}
        \| \compMomentum_{\AGGCall}-\bar{\Momentum}_S' \|^2 \leq \frac{\RobAvgCoeff}{|S|} \sum_{i\in S} \| \compMomentum_i-\bar{\Momentum}_S' \|^2,
    \end{align}
    where $\bar{\Momentum}_S'\defeq \frac{1}{|S|}\sum_{i\in S} \compMomentum_i =\frac{1}{|S|}\sum_{i\in S} \JLmatrix\Momentum_i=\JLmatrix\bar{\Momentum}_S$ and $\bar{\Momentum}_S\defeq \frac{1}{|S|}\sum_{i\in S} \Momentum_i$.
    % Thus, the new robustness parameter $\RobAvgCoeff'$ is obtained from
    Let decompression be $\decompressCall(\compMomentum_\AGGCall)=\JLmatrix^T \!\cdot\! \compMomentum_{\AGGCall}$.
    The goal is to prove: 
    \begin{align*}
        \| \JLmatrix^T \!\cdot\! \compMomentum_{\AGGCall}-\bar{\Momentum}_S \|^2 \leq \frac{\RobAvgCoeff'}{|S|} \sum_{i\in S} \| \Momentum_i-\bar{\Momentum}_S \|^2 + \RobAvgError\! \cdot \! \| \bar{\Momentum}_S \|^2 \!,
    \end{align*}
    with $\RobAvgCoeff'$ being close to the original coefficient $\RobAvgCoeff$, and $\RobAvgError$ being small.
    To evaluate the distortion in the decompressed output and find the robustness parameter $\RobAvgCoeff'$, we bound:
    \begin{align*}
    & \| \JLmatrix^T \!\cdot\! \compMomentum_{\AGGCall}-\bar{\Momentum}_S \|^2 \\
    % & = \| \JLmatrix^T \!\cdot\! \AGGCall(\compMomentum_1, \cdots, \compMomentum_\Totalclient)\!-\!\bar{\Momentum}_S+\JLmatrix^T \JLmatrix \bar{\Momentum}_S \!-\!\JLmatrix^T \JLmatrix \bar{\Momentum}_S\|^2\\
    & \stackrel{(a)}{\leq} \! \| \JLmatrix^T \!\cdot\! \compMomentum_{\AGGCall}\!-\!\JLmatrix^T \JLmatrix \bar{\Momentum}_S \|^2\!+\!\|\JLmatrix^T \JLmatrix \bar{\Momentum}_S \!-\!\bar{\Momentum}_S\|^2 \\
    & \stackrel{(b)}{\leq} \|\JLmatrix^T\|^2 \!\cdot\! \| \compMomentum_{\AGGCall}\!-\!\bar{\Momentum}_S' \|^2\!+\!\|\JLmatrix^T \JLmatrix \bar{\Momentum}_S \!-\!\bar{\Momentum}_S\|^2 \\
    & \stackrel{(c)}{\leq} \|\JLmatrix^T\|^2 \!\cdot\! \frac{\RobAvgCoeff}{|S|} \sum_{i\in S} \| \compMomentum_i-\bar{\Momentum}_S' \|^2+\|\JLmatrix^T \JLmatrix \bar{\Momentum}_S -\bar{\Momentum}_S\|^2 \\
    & \stackrel{(d)}{\leq} \|\JLmatrix^T\|^2 \!\cdot\! \|\JLmatrix\|^2 \!\cdot\! \frac{\RobAvgCoeff}{|S|} \sum_{i\in S} \| \Momentum_i-\bar{\Momentum}_S \|^2\!+\!\|\JLmatrix^T \JLmatrix \bar{\Momentum}_S \!-\!\bar{\Momentum}_S\|^2 \\
    & \stackrel{(e)}{\leq} (1+\JLepsilon)^2\cdot \frac{\RobAvgCoeff}{|S|} \sum_{i\in S} \| \Momentum_i-\bar{\Momentum}_S \|^2+ \JLepsilon^2 \| \bar{\Momentum}_S \|^2,
\end{align*}
where $(a)$ uses the triangle inequality, $(b)$ and $(d)$ use the sub-multiplicative property of the matrix norm, $(c)$ follows from \cref{eq:comp_kappa_averaging}, and $(e)$ follows from Property $\textit{3)}$ and $\textit{4)}$ in \cref{lem:jl_property}.

Therefore, $\compressCall\circ\AGGCall\circ\decompressCall$ with a JL matrix as the compressor is $(\Byzantine,\RobAvgCoeff', \RobAvgError)$-robust averaging, where $\RobAvgCoeff'=(1+\JLepsilon)^2\cdot \RobAvgCoeff$, $\RobAvgError=\JLepsilon^2$, and $\JLepsilon$ is the small distortion parameter introduced by the JL transform. Hence, any matrix satisfying the JL lemma yields a robust-compatible compression under the robust averaging criterion.
\end{proof}

% Since there is no known deterministic way for constructing JL matrices,
\cref{pro:JLwithRobAvg} states the robustness guarantees obtained by using any randomized construction that outputs a JL matrix with high probability, i.e., one that satisfies the properties of \cref{lem:jl_property}; see the construction in \cref{def:sparser_jl}.

\begin{observation}
    \label{pro:JLwithRobAvg}
    Consider a set $\mathcal{V}\subset \mathbb{R}^\oringinD$, $\compressD=\mathcal{O}(\JLepsilon^{-2}\log|\mathcal{V}|) \ll \oringinD$, and a random JL transform matrix $\JLmatrix \in \mathbb{R}^{\compressD\times \oringinD}$ that satisfies \cref{def:JL_lemma} with probability $1-\frac{1}{|\mathcal{V}|}$. This transform yields a robust-compatible compression under the robust averaging criterion with the same probability. 
Specifically, let $\AGGCall$ be a robust aggregation rule that is $(\Byzantine,\RobAvgCoeff)$-robust averaging. Let $\Momentum_i$ be the input vectors in the original space, and let the compressed vectors be $\compMomentum_i=\JLmatrix\Momentum_i$.
    % , and $S$ be any subset $S \in [\Totalclient]$ of size $\Totalclient-\Byzantine$. Let $S_1 \defeq \arg \min_{S: |S|=\Totalclient-\Byzantine} \left( \sum_{i\in S} \| \Momentum_i-\bar{\Momentum}_S \|^2 \right)$, $\bar{\Momentum}_S= \frac{1}{|S|}\sum_{i\in S} \Momentum_i$
    Then with probability $1-\frac{1}{|\mathcal{V}|}$, $\compressCall\circ\AGGCall\circ\decompressCall$ is a $(\Byzantine,\RobAvgCoeff',\RobAvgError)$-robust averaging rule with a robustness coefficient $\RobAvgCoeff'=(1+\JLepsilon)^2 \cdot \RobAvgCoeff$ and error $\RobAvgError=\JLepsilon^2$.
%     that satisfies 
%     \begin{align*}
% \mathbb{E}_{\JLmatrix}[\RobAvgCoeff'] \leq (1+\JLepsilon)^2 \RobAvgCoeff+\frac{| S |}{\sum_{i\in S_1} \| \Momentum_i-\bar{\Momentum}_{S_1} \|^2} \!\cdot\! \frac{\tau\oringinD}{\compressD} \| \bar{\Momentum}_S \|^2.
% \end{align*}
\end{observation}

\subsection{Privacy}
In \ourscheme, we enforce DP using the Gaussian mechanism~\cite{dwork2014algorithmic}, which achieves privacy by adding noise drawn from a Gaussian distribution with zero mean and variance $\DPsigma^2$, i.e.,
% \begin{align*}
    $\bm{n}_{\text{DP}} \sim \mathcal{N}(0, \DPsigma^2 I_\oringinD),$
% \end{align*}
where $I_\oringinD$ is the identity matrix of dimension $\oringinD \times \oringinD$ and $\DPsigma$ is the noise scale.

\subsection{Theoretical Guarantees}
% \tj{I revised the theorem please double check it.}
We formally state the theoretical guarantees of \ourscheme in terms of DP, convergence, and communication efficiency. The analysis requires the following standard assumptions~\cite{allouah2023privacy}:
\begin{assumption}
\label{ass:bounded}
We assume bounded variance and gradient norms for honest clients, i.e,
% We bound the heterogeneity and the variance of the gradients.
    \begin{enumerate}
        \item Bounded Variance: there exists parameters $0< \sigma_{\min} \leq \sigma_{\max} <\infty$ such that for each honest client $i\in \HonestSet$ and all $\model \in \mathbb{R}^\oringinD$, $\sigma_{\min}^2 \leq \frac{1}{\DatasetSize}\sum_{(\bm{x}_j,y_j) \in \Dataset_i} \| \nabla \SampleLoss(\model;\bm{x}_j,y_j)- \nabla \Loss(\model;\Dataset_i) \|^2 \leq \sigma_{\max}^2$. 
        % \yx{We need to assume a lower bound on the variance. If there is no lower bound, when all $\Momentum_i$ are the same, $\sum_{i\in S} \| \Momentum_i-\bar{\Momentum}_S \|^2=0$, $\kappa$' goes to infinity.}
        \item Bounded Gradient: there exists $0<C<\infty$ such that for all honest client $i\in \HonestSet$, and for all $\model \in \mathbb{R}^\oringinD$ and $(\bm{x}_j,y_j) \in \Dataset_i$, $\| \nabla \SampleLoss(\model;\bm{x}_j,y_j)\| \leq C.$
    \end{enumerate}
\end{assumption}

Similar to~\cite{allouah2023privacy}, we define $\Loss_0 \!\defeq\!\! \Loss_\HonestSet(\!\model^{(0)}\!)\!-\! \Loss^{\star}\!\!$, $\sigma_b\!=\!2(\!1\!-\!\frac{\MinibatchSize}{\DatasetSize}\!)\!\frac{\!\sigma_{\max}^2\!\!}{\MinibatchSize}\!,$
\begin{align*}
          \bar{\sigma}^2 & \defeq \frac{\sigma_b^2+d\DPsigma^2}{\Totalclient-\Byzantine}+4\RobAvgCoeff \left( \sigma_b^2 + 36\DPsigma^2(1+\frac{\oringinD}{\Totalclient-\Byzantine}) \right),\text{and } \\
          G_\mathrm{cov} &\defeq \sup_{\model \in \mathbb{R}^\oringinD} \sup_{\| v \|\leq 1}\sum_{i\in \HonestSet} \langle v, \nabla\Loss(\model;\Dataset_i)-\nabla\Loss_\HonestSet(\model) \rangle^2 / |\HonestSet|.         
    \end{align*}
    where $G_\mathrm{cov}$ is a metric for quantifying the heterogeneity between the gradients of honest clients' loss functions. 

\cref{pro:composition} provides the convergence guarantee for any robust aggregation that satisfies $(\Byzantine,\RobAvgCoeff,\RobAvgError)$-robust averaging (\cref{def:rob_avg_new}), with a constraint on $\RobAvgError$.

\begin{proposition}
\label{pro:composition}
    Consider a DP mechanism with noise scale $\DPsigma$ and parameters $\DPepsilon \geq 0, 0\leq \DPdelta \leq 1$. Suppose $\Loss_\HonestSet$ is $L$-smooth and the assumptions in \cref{ass:bounded} hold. Let $\AGGCall$ be a $(\Byzantine,\RobAvgCoeff)$-robust averaging rule, where $\Byzantine<\Totalclient/2$ and $\kappa>0$, $\compressCall$ be a robust-compatible compressor, and $\decompressCall$ be the corresponding decompressor. Suppose \ourframework with $\compressCall\circ\AGGCall\circ\decompressCall$ is $(\Byzantine,\RobAvgCoeff',\RobAvgError)$-robust averaging, $0 \leq \RobAvgError \leq 1$. We prove that the following holds. 
    \begin{enumerate}
        \item \textbf{\textit{Strongly convex}}: If $\Loss_\HonestSet$ is $\mu$-strongly convex, let  $\LearningRate^{(\GlobalIter)}=\frac{10}{\mu(\GlobalIter+a_1\frac{L}{\mu})}$, $\MomentumCoeff^{(\GlobalIter)}=1-24L\LearningRate^{(\GlobalIter)}$ and $a_1=240$. If we let $\RobAvgError \leq \frac{3}{125}$, then
        \begin{align*}
            \mathbb{E}[\Loss_\HonestSet(\model^{(\TotalGloablIter)})\!-\!\Loss^{\star}] \!\leq\! \frac{4a_1 \oringinD \RobAvgCoeff' G_\mathrm{cov}^2}{\mu}\!+\!\frac{2a_1^2L\bar{\sigma}^2}{\mu^2\TotalGloablIter}\!+\!\frac{2a_1^2L^2\Loss_0}{\mu^2\TotalGloablIter^2}.
        \end{align*}
        \item \textbf{\textit{Non-convex}}: Let $\LearningRate^{\GlobalIter} = \LearningRate= \min\{ \frac{1}{24L}, \frac{1}{8\bar{\sigma}} \sqrt{\frac{\Loss_0}{8 L \TotalGloablIter}} \}$ and $\MomentumCoeff^{(\GlobalIter)}= \MomentumCoeff =1-24L\LearningRate$. If we let $\RobAvgError < \frac{1}{15}$, then
        \begin{align*}
            & \mathbb{E} \left[ \nabla \Loss_{\mathcal{H}}(\tilde{\model}) \right] \\
            \leq & \frac{1}{A- 2C\RobAvgError} \left( \frac{27L \Loss_0}{\TotalGloablIter} + \frac{36 \bar{\sigma} \sqrt{2L \Loss_0}}{\sqrt{\TotalGloablIter}} + \frac{3}{2} \oringinD \RobAvgCoeff' G_{\text{cov}}^2 \right),
        \end{align*}
        where 
        $
        A \defeq \frac{1}{2} \left( (1-4\LearningRate L) - \frac{1}{2}(1+\LearningRate L) \MomentumCoeff^{2}  \right),
    $
    $
        C \defeq  1+\LearningRate L + \frac{1}{8} (1+ \LearningRate L) \MomentumCoeff^{2},
    $
    and $0 \! < \! A- 2C\RobAvgError \! \leq \! \frac{1}{2}$ always holds. 
    \end{enumerate}
\end{proposition}

\begin{proof}
    The proof is provided in \cref{proof:proposition3}.
\end{proof}

For \ourscheme, we prove the following theorem in terms of DP, convergence, and communication efficiency.
\begin{theorem}
\label{the:scheme}Consider a federated learning setting with $n$ clients out of which $b<n/2$ are malicious, where \ourscheme is used.
    Suppose $\Loss_\HonestSet$ is $L$-smooth and the assumptions in \cref{ass:bounded} hold. Consider a DP mechanism with noise scale $\DPsigma$ and parameters $\DPepsilon \geq 0, 0<\DPdelta <1$.  
    Let $\AGGCall$ be a $(\Byzantine,\RobAvgCoeff)$-robust averaging rule, $\compressCall$ be a random JL transform matrix $\JLmatrix\in\mathbb{R}^{k\times d}$, and $\decompressCall$ be $\JLmatrix^T$ . Let $\MinibatchSize$ and $\DatasetSize$ be the mini-batch size and dataset size, respectively. 
    The following guarantees are satisfied:
\begin{enumerate}
        \item \textbf{Differential Privacy}: There exists a constant $\chi>0$ such that, for a sufficiently small batch size $\MinibatchSize$ and the total global iteration $\TotalGloablIter>1$, if $\DPsigma~\geq~\chi \frac{2\ClipNorm}{\MinibatchSize} \max \left\{ 1,\frac{\MinibatchSize\sqrt{\TotalGloablIter \log(1/\DPdelta)}}{\DatasetSize \DPepsilon} \right\}$, \ourscheme satisfies $(\DPepsilon,\DPdelta)$-differential privacy;
        \item \textbf{Convergence}: Since with probability $1-\frac{1}{|\mathcal{V}|}$, the composition $\compressCall\circ\AGGCall\circ\decompressCall$ forms a $(\Byzantine,\RobAvgCoeff',\RobAvgError)$-robust averaging rule with a robustness coefficient $\RobAvgCoeff'=(1+\JLepsilon)^2 \cdot \RobAvgCoeff$ and error $\RobAvgError=\JLepsilon^2$,
         \ourscheme satisfies the convergence guarantees in \cref{pro:composition} with the same $\RobAvgCoeff'$ and $\RobAvgError$. To satisfy the required conditions on the error term $\RobAvgError$, we need
        $$
        0 \leq \JLepsilon \leq \sqrt{\frac{3}{125}} \approx 0.15,
        $$
        for the strongly convex case, and 
        % $\RobAvgError=\JLepsilon^2<\frac{1}{15}$ i.e., 
        $$
        0 \leq \JLepsilon < \sqrt{\frac{1}{15}} \approx 0.25,
        $$
        for the non-convex case.
     \item \textbf{Communication Efficiency}: 
     The bidirectional communication cost, i.e., from each client to the federator and from the federator to each client, is reduced from
     $\mathcal{O}(\oringinD)$ to $\mathcal{O}(\compressD)$ per round.
    \end{enumerate}

\end{theorem}
\begin{proof}
    The proof is provided in Appendix~\ref{app:proof}.
\end{proof}

\begin{remark}
Regarding convergence (robustness), a smaller $\JLepsilon$ means that
\begin{enumerate*}[label=(\emph{\alph*})]
\item the $\RobAvgCoeff'$ provided by the composition $\compressCall\circ\AGGCall\circ\decompressCall$ is closer to the $\RobAvgCoeff$ guaranteed by $\AGGCall$, i.e., the resulting scheme is as robust as $\AGGCall$; and
\item the additional error term $\RobAvgError$ is smaller.
\end{enumerate*}
This indicates a trade-off between communication efficiency and robustness. A larger compression rate results in a higher $\JLepsilon$, which in turn makes $\RobAvgCoeff'$ deviate further from the original $\RobAvgCoeff$ and yields a larger $\RobAvgError$. Conversely, a smaller compression rate increases the communication overhead but allows $\RobAvgCoeff'$ to remain closer to $\RobAvgCoeff$, and also leads to a smaller error term $\RobAvgError$.
\end{remark}

\begin{figure*}[tb]
    \centering
    \begin{minipage}{0.32\textwidth}
        \centering
        \includegraphics[width=0.85\linewidth, height=0.13\textheight]{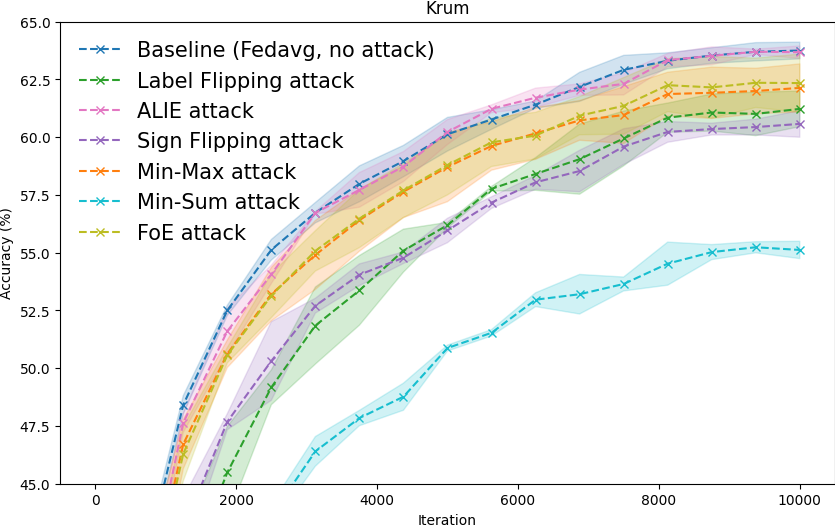}
    \end{minipage}%
    \begin{minipage}{0.32\textwidth}
        \centering
        \includegraphics[width=0.85\linewidth, height=0.13\textheight]{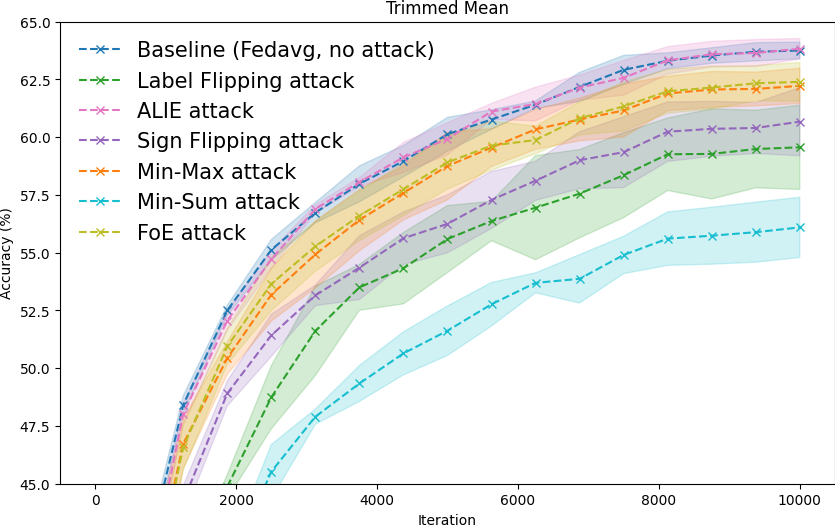}
    \end{minipage}
    \begin{minipage}{0.32\textwidth}
        \centering
        \includegraphics[width=0.85\linewidth, height=0.13\textheight]{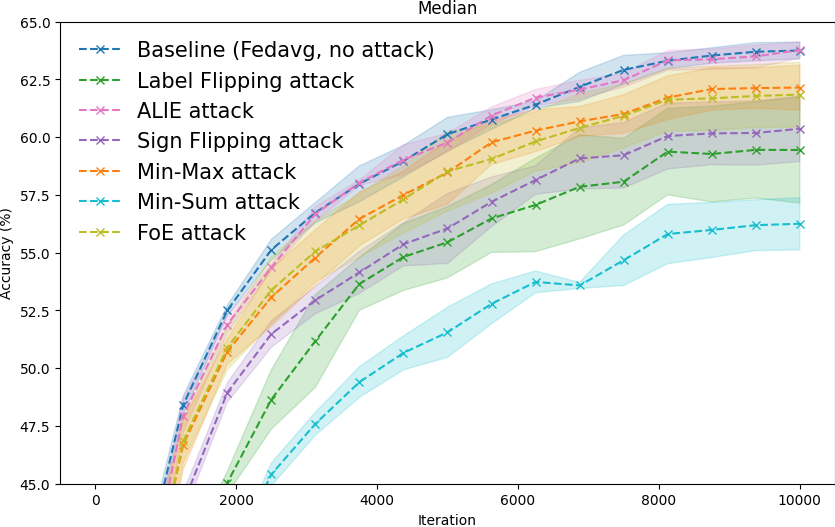}
    \end{minipage}
    \caption{Test accuracy (unit: \%) on non-i.i.d. CIFAR-10 under various attacks. We set $\NoiseMultiplier=0.1$ and compression rate $\frac{\oringinD}{\compressD}=10$.}
    \label{fig:cifar10}
\end{figure*}

\begin{table*}[tb] % 
\centering
% \small
\scalebox{0.9}{
\begin{tabular}{c|c!{\vrule width 0.8pt}c|c|c|c|c|c}
    
    & $\NoiseMultiplier$ & LF Attack & ALIE Attack & SF Attack & Min-Max Attack & Min-Sum Attack & FoE Attack \\ \hline
    \multirow{2}{*}{Krum} & $0.1$ & $83.6\pm0.5$ & $83.0\pm0.3$ & $82.1\pm0.1$ & $83.6\pm0.4$ & $75.1\pm2.6$ & $83.6\pm0.5$ \\
    & $1$ & $59.4\pm3.0$ & $73.2\pm0.4$ & $72.1\pm0.6$ & $73.1\pm0.5$ & $73.7\pm0.5$ & $73.6\pm0.7$ \\ \hline
    \multirow{2}{*}{TM} & $0.1$ & $83.6\pm0.5$ & $83.2\pm0.2$ & $82.4\pm0.2$ & $83.6\pm0.4$ & $75.7\pm2.6$ & $83.6\pm0.4$ \\
    & $1$ & $60.3\pm1.7$ & $73.2\pm0.4$ & $72.9\pm0.6$ & $73.4\pm0.5$ & $74.7\pm0.3$ & $73.6\pm0.8$ \\ \hline
    \multirow{2}{*}{Med} & $0.1$ & $83.6\pm0.5$ & $83.2\pm0.2$ & $82.4\pm0.2$ & $83.6\pm0.4$ & $75.1\pm2.5$ & $83.7\pm0.6$ \\
    & $1$ & $59.8\pm2.1$ & $73.1\pm0.4$ & $73.0\pm0.7$ & $73.5\pm0.5$ & $74.5\pm0.4$ & $73.6\pm0.7$
\end{tabular}}
\vspace{-0.1cm}
\caption{Test accuracy (unit: \%) on non-i.i.d. Fashion MNIST under various attacks and different noise multipliers. We set the compression rate $\frac{\oringinD}{\compressD}=10$. The baseline reaches an accuracy of $84.0\pm0.2\%$ for $\NoiseMultiplier=0.1$ and $75.8\pm0.2\%$ for $\NoiseMultiplier=1$.}
\label{tab:fmnist_nm}
\end{table*}

\begin{table*}[tb] % 
\centering
\scalebox{0.9}{
\begin{tabular}{c|c!{\vrule width 0.7pt}c|c|c|c|c|c}
    
    & Compression Rate & LF Attack & ALIE Attack & SF Attack & Min-Max Attack & Min-Sum Attack & FoE Attack \\ \hline
    \multirow{3}{*}{Krum} & $10$ & $61.2\pm0.8$ & $63.7\pm0.3$ & $60.6\pm0.6$ & $62.1\pm1.0$ &  $55.1\pm0.4$ & $62.3\pm1.3$ \\
    & $30$ & $59.6\pm1.5$ & $62.0\pm0.8$ & $59.2\pm0.9$ & $61.3\pm1.0$ & $54.1\pm0.5$ & $60.8\pm1.1$\\
    & $50$ & $57.3\pm1.0$ & $60.0\pm0.4$ & $57.5\pm0.7$ & $60.1\pm1.3$ & $53.4\pm0.5$ & $60.1\pm1.0$ \\ \hline
    \multirow{3}{*}{TM} & $10$ & $59.6\pm1.8$ & $63.8\pm0.4$ & $60.7\pm1.5$ & $62.2\pm0.8$ & $56.1\pm1.3$ & $62.4\pm0.8$ \\
    & $30$ & $58.7\pm1.0$ & $61.9\pm0.5$ & $59.5\pm1.1$ & $61.4\pm1.2$ & $55.4\pm0.6$ & $60.5\pm0.7$ \\
    & $50$ & $57.2\pm0.6$ & $60.3\pm0.4$ & $57.6\pm1.0$ & $60.2\pm1.3$ & $54.2\pm0.6$ & $60.0\pm1.0$ \\ \hline
    \multirow{3}{*}{Med} & $10$ & $59.4\pm2.3$ & $63.8\pm0.4$ & $60.4\pm1.4$ & $62.2\pm1.0$ & $56.3\pm1.1$ & $61.8\pm1.4$ \\
    & $30$ & $59.0\pm1.2$ & $62.1\pm0.7$ & $59.2\pm1.2$ & $61.5\pm1.1$ & $55.3\pm0.8$ & $61.0\pm0.9$ \\
    & $50$ & $57.2\pm0.8$ & $60.4\pm0.2$ & $57.8\pm1.3$ & $60.0\pm1.1$ & $54.3\pm0.7$ & $60.0\pm0.7$
\end{tabular}}
% \begin{tabular}{c|c!{\vrule width 0.8pt}c|c|c|c|c|c}
    
%     & \makecell{Compression \\ Rate} & LF Attack & ALIE Attack & SF Attack & Min-Max Attack & Min-Sum Attack & FoE Attack \\ \hline
%     \multirow{2}{*}{Krum} & $10$ & $61.2\pm0.8$ & $63.7\pm0.3$ & $60.6\pm0.6$ & $62.1\pm1.0$ &  $55.1\pm0.4$ & $62.3\pm1.3$ \\
%     & $50$ & $57.3\pm1.0$ & $60.0\pm0.4$ & $57.5\pm0.7$ & $60.1\pm1.3$ & $53.4\pm0.5$ & \\ \hline
%     \multirow{2}{*}{TM} & $10$ & $59.6\pm1.8$ & $63.8\pm0.4$ & $60.7\pm1.5$ & $62.2\pm0.8$ & $56.1\pm1.3$ & $62.4\pm0.8$ \\
%     & $50$ & $57.2\pm0.6$ & $60.3\pm0.4$ & $57.6\pm1.0$ & $60.2\pm1.3$ & $54.2\pm0.6$ & \\ \hline
%     \multirow{2}{*}{Med} & $10$ & $59.4\pm2.3$ & $63.8\pm0.4$ & $60.4\pm1.4$ & $62.2\pm1.0$ & $56.3\pm1.1$ & $61.8\pm1.4$ \\
%     & $50$ & $57.2\pm0.8$ & $60.4\pm0.2$ & $57.8\pm1.3$ & $60.0\pm1.1$ & $54.3\pm0.7$ & 
% \end{tabular}
% \vspace{-0.1cm}
\caption{Test accuracy (unit: \%) on non-i.i.d. CIFAR-10 under various attacks at compression rates $10$, $30$, and $50$. Baselines for these rates are $63.8\pm0.4\%$, $62.2\pm0.5\%$, and $61.7\pm0.2\%$, respectively. We set the noise multiplier $\NoiseMultiplier=0.1$.}
\label{tab:cifar_compress_rate}
\end{table*}

% \begin{table*}[tb] % 
% \centering
% % \small
% \begin{tabular}{c!{\vrule width 0.8pt}c|c|c|c|c|c}
    
%     & LF Attack & ALIE Attack & SF Attack & Min-Max Attack & Min-Sum Attack & FoE Attack \\ \hline
%     JL transform & $59.6\pm1.8$ & $63.8\pm0.4$ & $60.7\pm1.5$ & $62.2\pm0.8$ & $56.1\pm1.3$ & $62.4\pm0.8$ \\ \hline
%     top-$k$ sparsification & $53.0\pm1.8$ & $59.1\pm0.9$ & $55.3\pm1.1$ & $57.3\pm1.1$ & $47.6\pm0.7$ & $57.2\pm1.0$ 
% \end{tabular}
% % \vspace{-0.1cm}
% \caption{Comparison between using JL transform and top-$k$ sparsification as the compression method $\compressCall$ in terms of test accuracy (unit: \%) on non-i.i.d. CIFAR-10 under various attacks and use Trimmed Mean (TM) as the robust aggregation. We set the compression rate for both methods as $\frac{\oringinD}{\compressD}=10$. The baseline achieves an accuracy of $63.8\pm0.4\%$ with the JL transform and $59.6\pm0.4\%$ with top-$k$.}
% \label{tab:cifar_top_k}
% \end{table*}

\begin{table*}[!htbp] % 
\centering
\begin{tabular}{c!{\vrule width 0.8pt}c|c|c|c|c|c}
    $\NoiseMultiplier$ & LF Attack & ALIE Attack & SF Attack & Min-Max Attack & Min-Sum Attack & FoE attack \\ \hline
    $0.2$ & $54.6\pm0.1$ & $59.5\pm0.2$ & $56.8\pm1.1$ & $57.9\pm0.9$ & $53.5\pm0.9$ & $57.8\pm1.2$ \\ 
\end{tabular}
\vspace{-0.05cm}
\caption{We set noise multiplier $\NoiseMultiplier=0.2$. We report the test accuracy (unit: \%) on non-i.i.d. CIFAR-10 under various attacks and use Trimmed Mean (TM) as the robust aggregation. Baseline for noise multiplier $=0.2$ is $59.6\pm0.3\%$.}
\label{tab:cifar_noise_multiplier}
\end{table*}

% \begin{figure*}[!htbp]
%     \centering
%     \begin{minipage}{0.32\textwidth}
%         \centering
%         \includegraphics[width=0.8\linewidth, height=0.1\textheight]{figure/fmnist_comp10_krum_60.png}
%     \end{minipage}%
%     \begin{minipage}{0.32\textwidth}
%         \centering
%         \includegraphics[width=0.8\linewidth, height=0.1\textheight]{figure/fmnist_comp10_tm_60.png}
%     \end{minipage}
%     \begin{minipage}{0.32\textwidth}
%         \centering
%         \includegraphics[width=0.8\linewidth, height=0.1\textheight]{figure/fmnist_comp10_med_60.png}
%     \end{minipage}
%     \caption{non-i.i.d. Fashion MNIST. Compression rate $=10$. We set the noise multiplier $\NoiseMultiplier=0.1$. \yx{We can also report these in a table to save space.}}
%     \label{fig:Fashion MNIST}
% \end{figure*}

\section{Experiments}
We evaluate the practical performance of \ourscheme by conducting experiments on three datasets: Fashion MNIST~\cite{xiao2017fashion} and CIFAR-10~\cite{krizhevsky2009learning}, both distributed across $\Totalclient=15$ clients, and FEMNIST~\cite{caldas2018leaf} distributed across $\Totalclient=50$ clients.

\subsection{Training Parameters and Data Distribution}
For Fashion MNIST, we set minibatch size $\MinibatchSize=60$, total global iterations $\TotalGloablIter=2000$, and learning rate $\LearningRate=0.25$. We train a neural network consisiting of three fully connected layers (in total $535818$ parameters). 
For CIFAR-10, we set minibatch size $\MinibatchSize=128$, total global iterations $\TotalGloablIter=10000$, and learning rate $\LearningRate^{(\GlobalIter)}=0.25$ for $\GlobalIter \leq 8000$, and $\LearningRate^{(\GlobalIter)}=0.025$ afterward. We use a convolutional neural network composed of two convolutional blocks followed by a linear classifier (in total $789706$ parameters). 
For FEMNIST, we set minibatch size $\MinibatchSize=128$, total global iterations $\TotalGloablIter=1000$, and learning rate $\LearningRate^{(\GlobalIter)}=0.2$ for $\GlobalIter \leq 800$, and $\LearningRate^{(\GlobalIter)}=0.02$ afterward. We use a convolutional neural network consisting of two convolutional layers and a two-layer fully connected classifier (in total $1690046$ parameters).

All experiments are conducted with a fixed momentum coefficient $\MomentumCoeff=0.9$. 
We use the Rectified Linear Unit (ReLU) as the activation function and Cross-Entropy as the loss function. In each iteration, clients randomly sample a minibatch from their local training dataset using subsampling without replacement and perform local training on the minibatch.

To simulate non-IID data, for Fashion MNIST and CIFAR-10, we randomly divide the clients into $10$ groups. 
A training sample with label $j$ is assigned to group $j$ w.p. $a>0$, and to each of the remaining groups w.p. $\frac{1-a}{9}$. Within each group, data is uniformly distributed among the clients. In our experiments, we set $a=0.5$ to reflect a heterogeneous setting. 
FEMNIST is inherently heterogeneous. FEMNIST consists of 3500 clients with a varying number of samples per client, with a mean $226.83$. We filter the clients having at least $380$ samples, and randomly sample $\Totalclient=50$ clients from them.

\subsection{Compression Method}
For communication efficiency, we use the Count Sketch~\cite{kane2014sparser}, cf. \cref{def:sparser_jl}, to construct the JL transform matrix $\JLmatrix$. The number of block matrices is set to be $p=10$. For Fashion MNIST, the compression rate $\frac{\oringinD}{\compressD}=10$. Compression rates are varied across $\{10, 30, 50\}$ for CIFAR-10, and across $\{10, 50, 100, 500, 1000, 2000\}$ for FEMNIST.

\subsection{Byzantine Robustness}
We assume that $20\%$ of the clients are malicious, i.e., $\Byzantine=3$ for Fashion MNIST and CIFAR-10 and $\Byzantine=10$ for FEMNIST. Since \ourscheme supports arbitrary robust aggregation rules satisfying $(\Byzantine,\RobAvgCoeff)$-robust averaging, we test it with the following rules: Krum~\cite{blanchard2017machine}, Trimmed Mean (TM)~\cite{yin2018byzantine} and Median (Med)~\cite{yin2018byzantine}. These robust aggregations rules are proved to be $(\Byzantine,\RobAvgCoeff)$-robust averaging~\cite{allouah2023fixing, guerraoui2024robust}. We also use the Nearest Neighbor Mixing (NNM)~\cite{allouah2023fixing}, where 
each client's update is averaged with those of a subset of its neighbors to enhance the robustness under heterogeneity. Combining Krum, TM, or  Med with NNM improves the robustness coefficient $\RobAvgCoeff$ of the $(\Byzantine,\RobAvgCoeff)$-robust averaging rule \cite{allouah2023fixing}. 
The malicious clients perform one of the following attacks: Label Flipping (LF) attack~\cite{allen2020byzantine}, a-little-is-enough (ALIE) attack~\cite{baruch2019little}, Sign Flipping (SF) attack~\cite{allen2020byzantine}, Min-Max attack~\cite{shejwalkar2021manipulating}, Min-Sum attack~\cite{shejwalkar2021manipulating} and FoE attack~\cite{xie2020fall}.

\subsection{Differential Privacy}
To implement DP, we set the Gaussian noise scale as $\DPsigma=\frac{2\ClipNorm}{\MinibatchSize}\times \NoiseMultiplier$ with the clipping threshold $\ClipNorm=2$ for Fashion MNIST, $\ClipNorm=3$ for FEMNIST, and $\ClipNorm=4$ for CIFAR-10, and $\NoiseMultiplier$ is the noise multiplier. Since the DP guarantee $(\DPepsilon,\DPdelta)$ is sensitive to the choice of sampling strategy, e.g., Poisson sampling and sampling with(out) replacement, we report the noise multiplier $\NoiseMultiplier$ instead. 
We use Opacus~\cite{opacus} to compute the privacy budget $(\DPepsilon, \DPdelta)$ via RDP accounting.
Note that Opacus assumes Poisson subsampling, whereas our simulation uses sampling without replacement (fixed-size minibatches). Since Poisson subsampling leads to a looser privacy bound under RDP analysis than fixed-size sampling~\cite{wang2019subsampled}, the $\DPepsilon$ reported by Opacus is a conservative upper bound on the actual privacy loss incurred in our setting. 

%\yx{Unfortunately, we have a high privacy budget. (Is it better to put the paragraph below in the next subsection (experiment results))?} 
%\tj{In my opinion, this is the main concern of our paper, and the trade-off between privacy guarantees and utility does not arise from our proposed scheme. So it's better not to emphasize this point and avoid provoking the reviewers.}
For Fashion MNIST, we choose noise multipliers $\NoiseMultiplier=\{0.1,1\}$. The relatively large noise multiplier, i.e., $\NoiseMultiplier= 1$, corresponds to privacy budget $(\DPepsilon,\DPdelta)=(4.5,10^{-5})$. For CIFAR-10, we choose relatively small noise multipliers, e.g., $\NoiseMultiplier= \{0.1,0.2\}$, which correspond to large privacy budgets $\DPepsilon>300$ and $\DPdelta=10^{-5}$. For FEMNIST, we choose $\NoiseMultiplier=\{ 0.1,0.5,1 \}$, where $\NoiseMultiplier=1$ corresponds to $(\DPepsilon,\DPdelta)=(36.8,10^{-5})$. These configurations allow us to study the interaction between DP, robustness, and communication efficiency in a regime where utility is preserved. % Stronger privacy (lower $\DPepsilon$) can be achieved by increasing the noise multiplier
% %(or adjusting the model and training process), 
% at the cost of utility. For instance, $\NoiseMultiplier = 0.5$ yields $(\DPepsilon,\DPdelta)=(231.3,10^{-5})$ with $53.3\%$ test accuracy, and $\NoiseMultiplier = 1.0$ yields $(\DPepsilon,\DPdelta)=(35.2,10^{-5})$ but only $28.7\%$ test accuracy.

% We emphasize that our goal is not to minimize $\DPepsilon$ on its own, but to understand how differential privacy mechanisms influence other desirable properties in federated learning.  Exploring techniques to improve privacy without degrading accuracy is left for future work.

\begin{table*}[tb] % 
\centering
% \small
\begin{tabular}{c|c!{\vrule width 0.8pt}c|c|c|c|c}
    
    &  LF Attack & ALIE Attack & SF Attack & Min-Max Attack & Min-Sum Attack & FoE Attack \\ \hline
    Krum &  $75.0\pm1.62$ & $73.5\pm1.16$ & $70.0\pm1.13$ & $73.7\pm1.81$ & $57.0\pm1.10$ & $29.9\pm4.97^{\,\star}$ \\ \hline
    TM &  $75.2\pm2.27$ & $73.9\pm0.21$ & $70.4\pm1.12$ & $73.5\pm0.81$ & $59.9\pm0.25$ & $75.0\pm1.47$ \\ \hline
    Med & $75.3\pm2.38$ & $73.8\pm1.18$ & $70.2\pm1.14$ & $72.3\pm0.65$ & $59.5\pm0.49$ & $75.1\pm1.57$ 
\end{tabular}
\vspace{-0.1cm}
\caption{Test accuracy (unit: \%) on FEMNIST under various attacks and different robust aggregations. We set the compression rate $\frac{\oringinD}{\compressD}=10$ and $\NoiseMultiplier=0.1$. The baseline reaches an accuracy of $75.3\pm1.25 \%$. }
\captionsetup{justification=justified,parskip=1pt}
\vspace{-0.1in}
\caption*{\footnotesize $\star$ We also report the test accuracy on FEMNIST under FoE attack and with Krum as the robust aggregation rule without DP and compression ($\NoiseMultiplier=0$ and $\frac{\oringinD}{\compressD}=1$). The accuracy reached a maximum of $78.6\%$, but was ultimately defeated by the FoE attack, achieving a final accuracy of $6.26 \pm0.59\%$. } 
% This behavior is likely due to the strong heterogeneity of the FEMNIST dataset. Compared with the test accuracy reported in the table, where we set $\NoiseMultiplier=0.1$ and $\frac{\oringinD}{\compressD}=10$, we observe that under the same setting, the additional DP noise, together with the error introduced by compression, leads to consistently impaired performance rather than occasional catastrophic failures, highlighting the complex interaction between data heterogeneity, robustness mechanisms, and privacy constraints on FEMNIST.
% \yx{Can we focus on comparing the difference between $\NoiseMultiplier=0 \text{ and } \frac{\oringinD}{\compressD}=1$ and $\NoiseMultiplier=0.1 \text{ and } \frac{\oringinD}{\compressD}=10$?} \tj{It is helpful, but is there any paper that achieves this accuracy on the FEMNIST dataset under the FoE attack, using Krum and without DP, that we can refer to? By the way, I think it would be better to move this discussion to the appendix.}
\label{tab:femnist_aggrule}
\end{table*} 

\subsection{Numerical Results}
Each experiment is repeated using $3$ different random seeds. All baselines are run with the respective DP enforcement, i.e., noise multiplier, the compression rate, using FedAvg~\cite{mcmahan2017communication} as the aggregation rule, and without any malicious attacks.

\cref{fig:cifar10} shows the test accuracy of \ourscheme on non-i.i.d. CIFAR-10 under various attacks, compared to the baseline, with $\NoiseMultiplier=0.1$ and $\frac{\oringinD}{\compressD}=10$.
The experimental results in \cref{fig:cifar10} indicate that, given robust aggregation rules satisfying the robust averaging criterion, the use of the JL transform as the compression method preserves robustness against malicious attacks.
% when aggregating the JL transform as the compression method, they remain robust to malicious attacks.
\cref{tab:fmnist_nm} shows the test accuracy on non-i.i.d. Fashion MNIST under various attacks and noise multipliers. It demonstrates the robustness of using robust-compatible compression and shows that as the noise multiplier increases, utility decreases while stronger privacy guarantees are achieved. 
\cref{tab:cifar_compress_rate} summerizes the performance of different robust aggregation rules aggregating with count sketch under various attacks on non-i.i.d. CIFAR-10, at increasing compression rate, and noise multiplier $\NoiseMultiplier=0.1$. It suggests that a higher compression rate introduces greater distortion, and therefore, while improving communication efficiency, can lead to performance loss.
% In \cref{tab:cifar_top_k}, we substitute the JL transform with top-$k$ sparsification, cf. ~\cite{zhang2023byzantine}, which is not a robust-compatible aggregation, as it generally does not preserve $l_2$ norm. In the simulation, we send both the top-$k$ preserved entries and their corresponding positions to the federator, allowing the federator to decompress them before aggregation.
% 
\cref{tab:cifar_noise_multiplier} gives the test accuracy on non-i.i.d. CIFAR-10 under various attacks, with a slightly higher noise multiplier, i.e., $\NoiseMultiplier=0.2$, which also leads to a decrease in the utility.

\begin{figure}[!tb]
\vspace{-0.15in}
    \centering
    \includegraphics[width=0.67\linewidth]{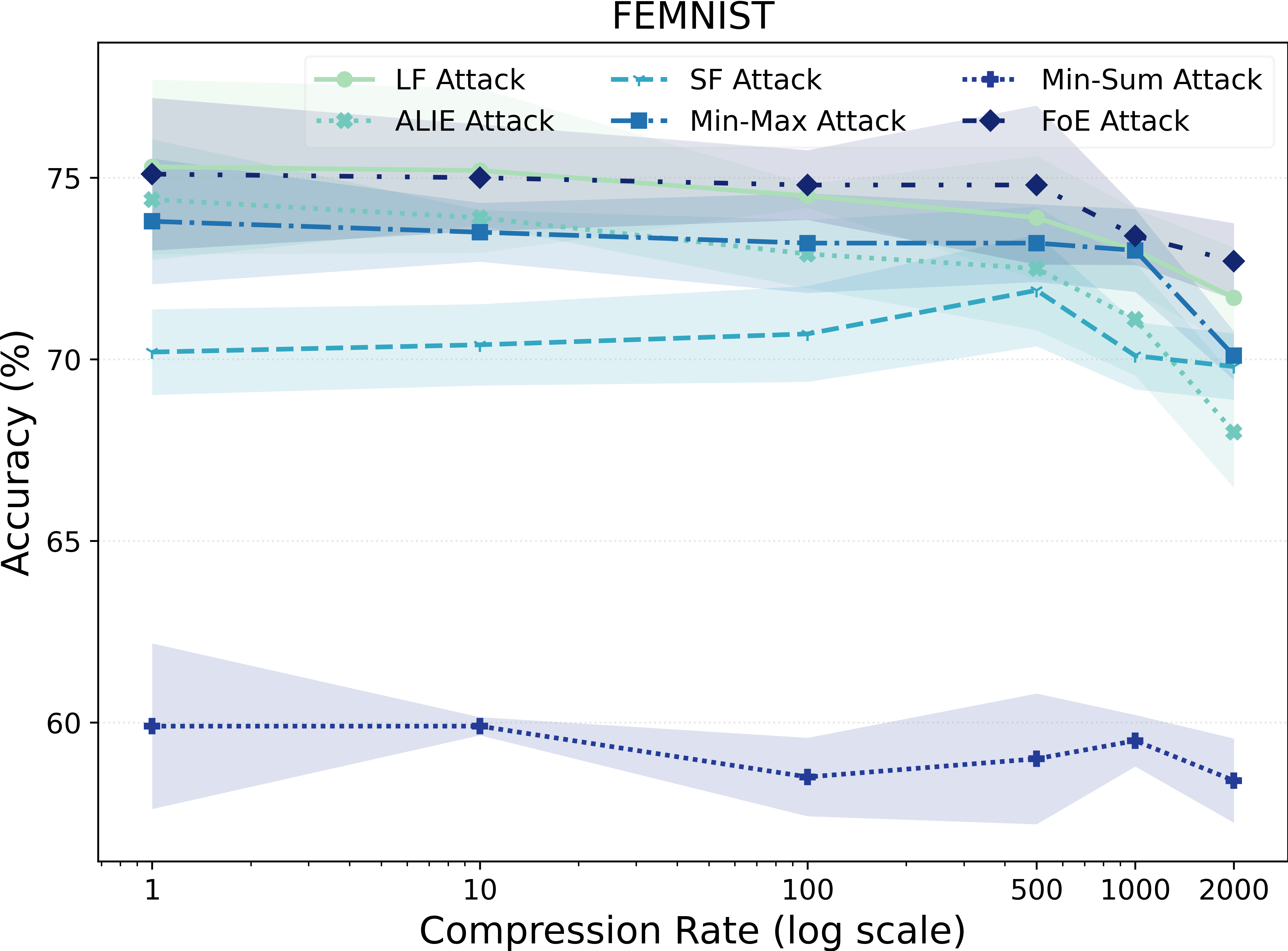}
    \caption{Test accuracy (unit: \%) on FEMNIST under various attacks and different compression rates $1,10,100,500,1000$ and $2000$. We set the noise multiplier $\NoiseMultiplier=0.1$ and use Trimmed Mean as the robust aggregation rule.}
    \label{fig:femnist_compression}
    \vspace{-0.07in}
\end{figure}

\begin{figure}[!tb]
\vspace{-0.05in}
    \centering
    \includegraphics[width=\linewidth]{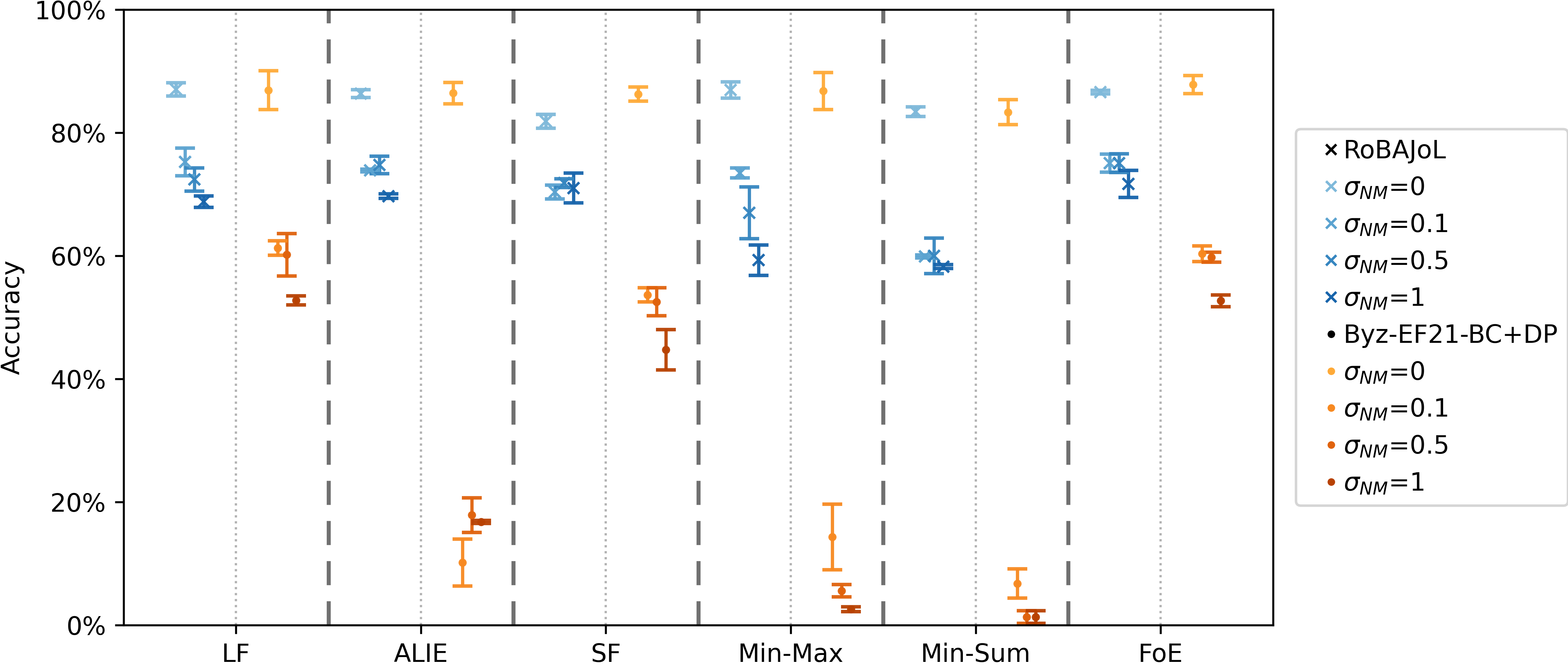}
    \caption{Test accuracy (unit: \%) on FEMNIST. We compare \ourscheme (on the left of each dotted line) with Byz-EF21-BC augmented with DP (on the right of each dotted line) under varying noise multipliers; Here, $\NoiseMultiplier=0$ indicates that no DP is applied. Both schemes use TM as the robust aggregation.}
    \label{fig:ef21}
    \vspace{-0.07in}
\end{figure}

We conduct simulations on FEMNIST under various attacks and across different compression rates, see \cref{fig:femnist_compression}, and across different robust aggregation rules, see \cref{tab:femnist_aggrule}.

% \yx{I finally got the results for the full FoE attack. I attached the results here. Reporting the results as a figure for all iterations would be beneficial in supporting the following argument. But I don't think we can afford any more figures.}
% \begin{figure}[H]
%     \centering
%     \includegraphics[width=0.5\linewidth]{figure/Krum_FoE.png}
%     \caption{Enter Caption}
%     \label{fig:placeholder}
% \end{figure}
% \yx{This behavior is likely due to the strong heterogeneity of the FEMNIST dataset and the fact that the FoE attack becomes powerful when the model approaches convergence~\cite{ergisi2025prodigy} (I asked Sena about this). The additional DP noise, together with the error introduced by compression, leads to consistently impaired performance rather than catastrophic failures when the model is approaching convergence, highlighting the complex interaction between data heterogeneity, robustness mechanisms, and privacy constraints on FEMNIST.}

We also compare \ourscheme with Byz-EF21-BC~\cite{rammal2024communication},  
to which we additionally incorporate differential privacy. Byz-EF21-BC allows any robust aggregation combined with any biased compressor (we use top-$k$ sparsification in the simulation) to remain robust by carefully modifying the updates. To allow a fair comparison, we clip the gradients and add noise directly to the averaged gradient, and use the NNM before aggregation to enhance robustness under heterogeneity. The comparison with varying noise multipliers is shown in \cref{fig:ef21}, demonstrating that, even though Byz-EF21-BC keeps any robust aggregation robust when combining with the biased compressor, i.e., top-$k$ sparsification ($\NoiseMultiplier=0$), however, enforcing DP is not straightforward and leads to a large drop in accuracy ($\NoiseMultiplier=\{ 0.1,0.5,1 \}$).

\section{Conclusion}
In this work, we propose \ourframework, a unified approach ensuring differential privacy, Byzantine robustness, and communication efficiency. We introduce the concept of robust-compatible compression, instantiated as \ourscheme, which uses the JL transform for compression, robust averaging to guarantee robustness, and the Gaussian DP mechanism to ensure privacy. Theoretical analysis and experiments show that our approach preserves robustness and privacy while significantly reducing communication overhead.

% \clearpage

% \small
% \renewcommand{\baselinestretch}{0.95}
% \IEEEtriggeratref{1}
\bibliographystyle{ieeetr}  
\bibliography{cleaned_refs}

% \clearpage
\appendix
\subsection{Proof \cref{lem:jl_property}}
\label{proof:lemma1}
% Here we provide some useful properties of the JL Transform.

% \begin{lemma}
% \label{lem:jl_property}
% Let matrix $\JLmatrix$ be a linear map satisfying Johnson-Lindenstrauss lemma. Then we have
% \begin{enumerate}
%     \item $\| \JLmatrix\bm{v} \|^2$ is an unbiased estimator of $\| \bm{v} \|^2$~\cite{vempala2005random,achlioptas2003database, kane2014sparser}, i.e.,
%     \begin{align*}
%         \mathbb{E}_{\JLmatrix}[\| \JLmatrix\bm{v} \|^2]=\| \bm{v} \|^2;
%     \end{align*}
%     \item for any $\bm{v} \in \mathbb{R}^\oringinD$~\cite{song2021iterative, chen2022fundamental}, 
%     \begin{align*}
%         \mathbb{E}_{\JLmatrix}[\| \JLmatrix^T\JLmatrix\bm{v} -\bm{v}\|^2] \leq \frac{\tau\oringinD}{\compressD} \| \bm{v} \|^2,
%     \end{align*}
%     where $\tau$ is a small constant, typically in the range  $\tau\in[1,3]$ depending on the distribution of $\JLmatrix$; and
%     \item the matrix $\JLmatrix$ satisfies
%     \begin{align*}
%         \| \JLmatrix \|^2 = \| \JLmatrix^T \|^2 \leq (1+\JLepsilon)^2.
%     \end{align*}
% \end{enumerate}
    
% \end{lemma}
% \begin{proof} 
Since the first and second properties are proved in the references, we prove the third and fourth properties of \cref{lem:jl_property}.

    \textit{(i) Bound on the decompression error: } We first rewrite the decompression error as
    \begin{align*}
        \| \JLmatrix^T \JLmatrix \bm{v} - \bm{v} \|^2 & = \| (\JLmatrix^T \JLmatrix - I_\oringinD) \bm{v} \|^2 \leq \| \JLmatrix^T \JLmatrix - I_\oringinD \|^2 \cdot \| \bm{v} \|^2, 
    \end{align*} 
    where $I_\oringinD \in \mathbb{R}^{\oringinD \times \oringinD}$ is an identity matrix. For a random JL transform matrix $\JLmatrix$, which satisfies JL lemma (\cref{def:JL_lemma}) w.p. $1-\frac{1}{|\mathcal{V}|}$, we have
    $
        1-\JLepsilon \leq \frac{\bm{v}^T \JLmatrix^T \JLmatrix \bm{v}}{\bm{v}^T \bm{v}} \leq 1+\JLepsilon
    $
    with the same probability. Thanks to Rayleigh quotient~\cite{horn2012matrix}, the middle term of this inequality is lower bounded by the smallest eigenvalue of $\JLmatrix^T \JLmatrix$ and upper bounded by its largest eigenvalue. Thus, we have
    $
        1-\JLepsilon \leq \lambda_i \leq 1+\JLepsilon,
    $
    where $\lambda_i$'s are the eigenvalues of the matrix $\JLmatrix^T \JLmatrix$, w.p. $1-\frac{1}{|\mathcal{V}|}$. Therefore, the eigenvalues of the matrix $\JLmatrix^T \JLmatrix - I_\oringinD$ are in $[-\JLepsilon, \JLepsilon]$, and the eigenvalues of the matrix $(\JLmatrix^T \JLmatrix - I_\oringinD)^2$ are in $[0, \JLepsilon^2]$ w.p. $1-\frac{1}{|\mathcal{V}|}$. Since the spectral norm of a matrix $\bm{M}$ is the square root of the largest eigenvalue of the matrix $\bm{M}^T\bm{M}$~\cite{meyer2023matrix}, we have
    \begin{align*}
        \| \JLmatrix^T \JLmatrix - I_\oringinD \| & = \sqrt{\lambda_{\max}\left((\JLmatrix^T \JLmatrix - I_\oringinD)^T(\JLmatrix^T \JLmatrix - I_\oringinD)\right)} 
        % \\
        % & = \sqrt{\lambda_{\max}( \JLmatrix^T \JLmatrix - I_\oringinD )^2} 
        \leq \JLepsilon. 
    \end{align*}
    Therefore, 
    % \begin{align*}
       $ \| \JLmatrix^T \JLmatrix \bm{v} - \bm{v} \|^2 \leq \JLepsilon^2 \| \bm{v} \|^2$
    % \end{align*} 
    holds w.p. $1-\frac{1}{|\mathcal{V}|}$.
    
    \textit{(ii) Bound on the squared spectral norm:} We aim to bound the spectral norm of the matrix $\JLmatrix$ and its transpose $\JLmatrix^T$, i.e., the matrix norm induced by the Euclidean vector norm, denoted $\| \JLmatrix \|$ and $\| \JLmatrix^T \|$. 
    % Since the spectral norm of a given matrix $\JLmatrix$ is the square root of the largest eigenvalue of the matrix $\JLmatrix^T\JLmatrix$~\cite{meyer2023matrix}, i.e.,
    We can express the spectral norm~\cite{meyer2023matrix} as
    % \begin{align*}
       $ \| \JLmatrix \| =\sqrt{\lambda_{\max}(\JLmatrix^T \JLmatrix)} = \sqrt{\lambda_{\max}(\JLmatrix \JLmatrix^T)} = \| \JLmatrix^T \|$,
    % \end{align*}
    where the second equality holds because the matrices $\JLmatrix^T \JLmatrix$ and $\JLmatrix\JLmatrix^T$ share the same non-zero eigenvalues.
    
    % Consider a $d \times d$ symmetric matrix $M$ with eigenvalues $\lambda_1 \leq \cdots \leq \lambda_d$. According to the Courant-Fisher theorem~\cite{courant1962methods}, we know that 
    % \begin{align*}
    %         \lambda_i = \min_{S\subseteq \mathbb{R}^\oringinD \atop \dim(S)=i} \max_{\bm{x}\in S \atop \bm{x}\neq\bm{0}}\frac{\bm{x}^TM\bm{x}}{\bm{x}^T\bm{x}},
    % \end{align*}
    % where the maximization and minimization are over subspaces $S$. For the largest eigenvalue, i.e., $i=d$, we have only one possible subspace with $\dim(S)=d$, which is the space $\mathbb{R}^d$ itself. Thus, the largest eigenvalue is 
    % \begin{align*}
    %         \lambda_d(M) = \min_{S\subseteq \mathbb{R}^n \atop \dim(S)=d} \max_{\bm{x}\in S \atop \bm{x}\neq\bm{0}}\frac{\bm{x}^TM\bm{x}}{\bm{x}^T\bm{x}} = \max_{\bm{x}\in \mathbb{R}^d \atop \bm{x}\neq\bm{0}}\frac{\bm{x}^TM\bm{x}}{\bm{x}^T\bm{x}}.
    % \end{align*}
    % In our case, the symmetric matrix $M=\JLmatrix^T\JLmatrix$, thus, we have
    % \begin{align*}
    %         \lambda_d(\JLmatrix^T\JLmatrix) = \max_{\bm{x}\in \mathbb{R}^d \atop \bm{x}\neq\bm{0}}\frac{\bm{x}^T \JLmatrix^T\JLmatrix \bm{x}}{\bm{x}^T\bm{x}} = \max_{\bm{x}\in \mathbb{R}^d \atop \bm{x}\neq\bm{0}} \frac{\| \JLmatrix\bm{x} \|^2}{\| \bm{x} \|^2} \leq 1+\JLepsilon,
    % \end{align*}
    % where the last inequality is satisfied with high probability, due to the JL lemma. 
    Since the eigenvalues of $\JLmatrix^T \!\JLmatrix$ are upper bounded by $1\!+\!\JLepsilon$ w.p. $1\!-\!\frac{1}{|\mathcal{V}|}$, thus
    % \begin{align*}
       $ \| \JLmatrix^T \|^2 \!\!=\!\! \| \JLmatrix \|^2 \!\!=\!\!\lambda_{\max}(\JLmatrix^T \JLmatrix)\!\leq\!1\!+\!\JLepsilon$
    % \end{align*}
    w.p. $1\!-\!\frac{1}{|\mathcal{V}|}$, which concludes the proof.
% \end{proof}

\subsection{Proof \cref{pro:composition}}
\label{proof:proposition3}
% \begin{proof}
In the convergence proof of~\cite{allouah2023privacy}, the authors analyze the algorithm with aggregation $\AGGCall$, where the inputs are $\Momentum_i^{(t)}$ for $i \in {n}$, and the output is $\Psi^{(\GlobalIter)}=\AGGCall(\Momentum_1^{(t)},\cdots,\Momentum_n^{(t)})$, and $\AGGCall$ is $(\Byzantine,\RobAvgCoeff)$-robust averaging. In \ourscheme, the composition $\compressCall \circ \AGGCall \circ \decompressCall$ provides a $(\Byzantine,\RobAvgCoeff,\RobAvgError)$-robust averaging aggregation rule, with a different $\RobAvgCoeff$ value.

\subsubsection{Strongly convex}
% First, we prove the convergence for the strongly convex setting. 
% \begin{proof}
    Assume that $\Loss_{\mathcal{H}}$ is $L$-smooth and $\mu$-strongly convex, and the composition $\compressCall \circ \AGGCall \circ \decompressCall$ is $(\Byzantine,\RobAvgCoeff,\RobAvgError)$-robust averaging. Throughout the proof, we adopt the same definitions of the learning rate and the momentum parameter, i.e., $\LearningRate^{(\GlobalIter)}=\frac{10}{\mu(\GlobalIter+a_1 \frac{L}{\mu})}$, and $\MomentumCoeff = 1-24L \LearningRate^{(\GlobalIter)}$, where $a_1=240$, as~\cite{allouah2023privacy}. We also have $\LearningRate^{(\GlobalIter)} \leq \LearningRate^{(0)} =\frac{10}{240 \mu \frac{L}{\mu}}=\frac{1}{24L}$. 
    
    We bound the error $\mathcal{E}^{(\GlobalIter)}$ between the output $\Psi^{(\GlobalIter)}=\JLmatrix^T \AGGCall(\compMomentum^{(t)},\cdots,\compMomentum^{(t)})$ and the average honest momentum $\bar{\Momentum}^{(\GlobalIter)} \defeq \frac{1}{|\mathcal{H}|}\sum_{i \in \mathcal{H}}{\Momentum_i^{(\GlobalIter)}}$, i.e., $\mathcal{E}^{(\GlobalIter)} \defeq \Psi^{(\GlobalIter)}-\bar{\Momentum}^{(\GlobalIter)}$. Define $\Delta^{(\GlobalIter)}=\lambda_{\max} \left( \frac{1}{|\mathcal{H}|} \sum_{i\in \mathcal{H}} ( \Momentum_i^{(\GlobalIter)}-\bar{\Momentum}^{(\GlobalIter)} ) ( \Momentum_i^{(\GlobalIter)}-\bar{\Momentum}^{(\GlobalIter)} )^T \right)$, where $\lambda_{\max}$ is the maximum eigenvalue, and $\delta^{(\GlobalIter)} \defeq \bar{\Momentum}^{(\GlobalIter)} - \nabla\Loss_\HonestSet(\model^{(\GlobalIter)})$.
    Since 
    \begin{align}
    \label{eq:trace}
    & \frac{\RobAvgCoeff}{|S|} \sum_{i\in S} \| \Momentum_i^{(\GlobalIter)}- \bar{\Momentum}^{(\GlobalIter)} \|^2 \\
    \nonumber = & \text{tr} \left( \frac{1}{|\mathcal{H}|} \sum_{i\in \mathcal{H}} ( \Momentum_i^{(\GlobalIter)}-\bar{\Momentum}^{(\GlobalIter)} ) ( \Momentum_i^{(\GlobalIter)}-\bar{\Momentum}^{(\GlobalIter)} )^T \right) \leq  \oringinD \Delta^{(\GlobalIter)},
    \end{align}
    where the last inequality is due to the trace of a matrix being the sum of its eigenvalues. We can bound the error between the output and the average honest momentum,
    \begin{align}
    \label{eq:error}
        & \mathbb{E}[\| \mathcal{E}^{(\GlobalIter)} \|^2] =  \mathbb{E}\left[ \| \Psi^{(\GlobalIter)}-\bar{\Momentum}^{(\GlobalIter)} \|^2 \right] \\ \nonumber
        \stackrel{(a)}{\leq} & \mathbb{E} \left[ \frac{\RobAvgCoeff}{|\mathcal{H}|} \sum_{i\in \mathcal{H}} \| \Momentum_i^{(\GlobalIter)}-\bar{\Momentum}^{(\GlobalIter)} \|^2+ \RobAvgError \| \bar{\Momentum}^{(\GlobalIter)}\|^2 \right] \\ \nonumber
        \stackrel{(b)}{\leq} & \oringinD \RobAvgCoeff \mathbb{E} \left[ \Delta^{(\GlobalIter)}\right] + \RobAvgError \mathbb{E} \left[  \| \bar{\Momentum}^{(\GlobalIter)}\|^2\right] \\ \nonumber
        \stackrel{(c)}{\leq} & \oringinD \RobAvgCoeff \mathbb{E} \left[ \Delta^{(\GlobalIter)}\right] + 2\RobAvgError \mathbb{E} \left[  \| \delta^{(\GlobalIter)} \|^2\right] + 2\RobAvgError \mathbb{E} \left[  \| \nabla\Loss_\HonestSet(\model^{(\GlobalIter)}) \|^2\right],
    \end{align}
    where $(a)$ uses \cref{def:rob_avg_new}, $(b)$ is obtained by plugging \cref{eq:trace} and rearranging the terms, $(c)$ is because of the definition of $\delta^{(\GlobalIter)}$, and the fact that $\| \delta^{(\GlobalIter)}+\nabla\Loss_\HonestSet(\model^{(\GlobalIter)}) \|^2 \leq 2  \| \delta^{(\GlobalIter)} \|^2 + 2 \| \nabla\Loss_\HonestSet(\model^{(\GlobalIter)}) \|^2 $.
    As in~\cite{allouah2023privacy}, we define $$V^{(\GlobalIter)} \!\! \defeq \!\! \left(\! t+a_1\frac{L}{\mu} \!\right)^2\!\! \mathbb{E} \!\left[ \Loss_\HonestSet(\model^{(\GlobalIter)}\!)\!- \!\Loss^{\star} \!+\! \frac{z_1}{L} \| \delta^{(\GlobalIter)} \|^2 \!\!+\! \oringinD\RobAvgCoeff \cdot \frac{z_2}{L}\Delta^{(\GlobalIter)} \!\right]\!,$$ where $z_1=\frac{1}{16}$, $z_2=2$, $$
    W^{(\GlobalIter)} \defeq \mathbb{E} \left[ \Loss_\HonestSet(\model^{(\GlobalIter)})- \Loss^{\star} \!+\! \frac{z_1}{L} \| \delta^{(\GlobalIter)} \|^2 \!+\! \oringinD\RobAvgCoeff \cdot \frac{z_2}{L}\Delta^{(\GlobalIter)} \right],$$
    $\bar{\sigma}_{b}^2 \defeq 2(1-\frac{\MinibatchSize}{\DatasetSize})\frac{\sigma_{\max}^2}{\MinibatchSize}$, $\bar{\sigma}_{DP}^2 \defeq \bar{\sigma}_{b}^2+\oringinD \cdot \DPsigma^2$, and $\bar{\sigma} \defeq \frac{\sigma_b^2 + d \sigma_{DP}^2}{\Totalclient-\Byzantine}+4\oringinD \RobAvgCoeff \left( \sigma_b^2 + 36 \sigma_{DP}^2 (1+\frac{\oringinD}{\Totalclient-\Byzantine}) \right)$. We also define
    $$
        A \defeq \frac{1}{2} \left( (1-4\LearningRate^{(\GlobalIter)} L) - 8z_1(1+\LearningRate^{(\GlobalIter)}L) \MomentumCoeff^{(\GlobalIter)^2}  \right),
    $$
    $$
        B \!\defeq \!\!  -\frac{1}{z_1} \!(1+2 \LearningRate^{(\GlobalIter)} L)-\frac{1}{\LearningRate^{(\GlobalIter)} L} \MomentumCoeff^{(\GlobalIter)^2} \!\!(1+5\LearningRate^{(\GlobalIter)} L \! + \! 4 \LearningRate^{(\GlobalIter)^2} L^2) + \frac{1}{\LearningRate^{(\GlobalIter)} L},
    $$
    and
    $
        C \defeq  1+\LearningRate^{(\GlobalIter)} L + 2 z_1 (1+ \LearningRate^{(\GlobalIter)} L) \MomentumCoeff^{(\GlobalIter)^2}.
    $
    We denote $\hat{t}=t+a_1\frac{L}{\mu}$. Thus, we have
    \begin{align}
    \label{eq:Wdiff}
        \nonumber & V^{(\GlobalIter+1)} \!\!-\!V^{(\GlobalIter)} \!\!= (\hat{t}+\!1)^2(W^{(\GlobalIter+1)} \!\!-\! W^{(\GlobalIter)}) \!+\! \left(\!2\hat{t}\!+\!1\!\right)W^{(\GlobalIter)}, \text{and} \\
        & W^{(\GlobalIter+1)} \!-\! W^{(\GlobalIter)} \\ \nonumber
        \stackrel{(a)}{\leq} & -\LearningRate^{(\GlobalIter)} A \cdot \mathbb{E} \left[  \| \nabla\Loss_\HonestSet(\model^{(\GlobalIter)}) \|^2\right] - z_1 \LearningRate^{(\GlobalIter)} B \cdot \mathbb{E}\left[\! \| \delta^{(\GlobalIter)} \|^2 \!\right] \\ \nonumber
        + & \LearningRate^{(\GlobalIter)} C \cdot \mathbb{E}\left[ \| \mathcal{E}^{(\GlobalIter)} \|^2 \right] + \oringinD \RobAvgCoeff \frac{z_2}{L} (1-\MomentumCoeff^{(\GlobalIter)}) G_{\text{cov}}^2 \\ \nonumber
        - & \oringinD \RobAvgCoeff \frac{z_2}{L} (1-\MomentumCoeff^{(\GlobalIter)}) \mathbb{E}[\Delta^{(\GlobalIter)}] + \frac{z_1}{L}(1-\MomentumCoeff^{(\GlobalIter)})^2\frac{\bar{\sigma}_{DP}^2}{\Totalclient-\Byzantine} \\ \nonumber
        + & 2 \oringinD \RobAvgCoeff \frac{z_2}{L} (1-\MomentumCoeff^{(\GlobalIter)})^2 \left( \bar{\sigma}_{b}^2 + 36\DPsigma^2(1+\frac{\oringinD}{\Totalclient-\Byzantine})  \right) \\ \nonumber
        % \stackrel{(b)}{\leq} & -\LearningRate^{(\GlobalIter)} A \cdot \mathbb{E} \left[  \| \nabla\Loss_\HonestSet(\model^{(\GlobalIter)}) \|^2\right]- \!z_1 \LearningRate^{(\GlobalIter)} B \cdot \mathbb{E}\left[\! \| \delta^{(\GlobalIter)} \|^2 \!\!\right] \\ \nonumber
        % + & \LearningRate^{(\GlobalIter)} C \cdot \oringinD \RobAvgCoeff \mathbb{E} \left[ \Delta^{(\GlobalIter)}\right] + \LearningRate^{(\GlobalIter)}  C \cdot 2\RobAvgError \mathbb{E} \left[  \| \delta^{(\GlobalIter)} \|^2\right] \\ \nonumber
        % + & \LearningRate^{(\GlobalIter)} C \cdot 2\RobAvgError \mathbb{E} \left[  \| \nabla\Loss_\HonestSet(\model^{(\GlobalIter)}) \|^2\right] + \oringinD \RobAvgCoeff \frac{z_2}{L} (1-\MomentumCoeff^{(\GlobalIter)}) G_{\text{cov}}^2 \\ \nonumber
        % - & \oringinD \RobAvgCoeff \frac{z_2}{L} (1-\MomentumCoeff^{(\GlobalIter)}) \mathbb{E}[\Delta^{(\GlobalIter)}] + \frac{z_1}{L}(1-\MomentumCoeff^{(\GlobalIter)})^2\frac{\bar{\sigma}_{DP}^2}{\Totalclient-\Byzantine} \\ \nonumber
        % + & 2 \oringinD \RobAvgCoeff \frac{z_2}{L} (1-\MomentumCoeff^{(\GlobalIter)})^2 \left( \bar{\sigma}_{b}^2 + 36\DPsigma^2(1+\frac{\oringinD}{\Totalclient-\Byzantine})  \right) \\ \nonumber
        \stackrel{(b)}{\leq} & \!-\!\LearningRate (\!A\! - \!2C\RobAvgError\!)\! \cdot \! \mathbb{E}\!\! \left[ \! \| \!\nabla\!\Loss_\HonestSet\!(\model^{(\!\GlobalIter)}\!)\! \|^2 \! \right] \!\!-\!\!z_1 \!\LearningRate (\!B\! -  \!\frac{2C \RobAvgError}{z_1} \! ) \!\!\cdot \! \mathbb{E} \!\left[\! \| \delta^{(\!\GlobalIter)} \! \|^2 \!\right] \\ \nonumber
        & \!\!\!- \!\oringinD \RobAvgCoeff z_2 \LearningRate \!\left(\!\! \frac{1-\MomentumCoeff}{\LearningRate L} \!-\! \frac{C}{z_2}\! \right)\! \!\cdot \!\mathbb{E}\!\! \left[ \!\Delta^{\!(\GlobalIter)}\!\right] \!\!+\! \frac{\!(1-\MomentumCoeff)^2\!\!}{L} \bar{\sigma}^2 \!\!\!+\! \oringinD \RobAvgCoeff \frac{z_2}{L} \!(1\!-\!\MomentumCoeff) G_{\text{cov}}^2 \\ \nonumber
        \stackrel{(c)}{\leq} & -2\LearningRate^{(\GlobalIter)} \mu \mathbb{E} [ (A-2C\RobAvgError) \cdot (\Loss_\HonestSet(\model^{(\GlobalIter)})-\Loss^{\star}) \\ \nonumber
         \;\; +& \frac{1}{2} (B\hspace{-1mm}-\frac{2}{z_1}C\RobAvgError) \hspace{-1mm}\cdot \frac{z_1}{L} \| \delta^{(\GlobalIter)} \|^2 
         \;\;\hspace{-3mm}+ \frac{1}{2}(\frac{1-\MomentumCoeff^{(\GlobalIter)}}{\LearningRate^{(\GlobalIter)}L}\hspace{-1mm}-\frac{1}{z_2}C) \frac{\oringinD \RobAvgCoeff z_2}{L} \Delta^{(\GlobalIter)} ] \\ \nonumber
         + &\frac{1}{L}(1-\MomentumCoeff^{(\GlobalIter)})^2 \bar{\sigma}^2 + \oringinD \RobAvgCoeff \frac{z_2}{L} (1-\MomentumCoeff^{(\GlobalIter)}) G_{\text{cov}}^2,
    \end{align}
    where $(a)$ is from~\cite{allouah2023privacy}, $(b)$ uses \cref{eq:error} to substitute $\mathbb{E}[\| \mathcal{E}^{(\GlobalIter)} \|^2]$, rearranges the terms, and uses $\bar{\sigma}^2 \geq z_1\frac{\bar{\sigma}_{DP}^2}{\Totalclient-\Byzantine} + 2 \oringinD \RobAvgCoeff z_2 \left( \bar{\sigma}_{b}^2 + 36\DPsigma^2(1+\frac{\oringinD}{\Totalclient-\Byzantine})  \right)$, $(c)$ comes from the fact that $\Loss_{\mathcal{H}}$ is $\mu$-strongly convex, i.e., for any $\model \in \mathbb{R}^d$, $\| \nabla \Loss_{\mathcal{H}}(\model) \|^2 \geq 2 \mu (\Loss(\model)-\Loss^{\star})$ and $L \geq \mu$.
    
    We want to make sure that the following constraints hold:
    \begin{enumerate*}[label=(\emph{\alph*})]
        \item $A-2C\RobAvgError>0$;
        \item $\frac{1/2(B-\frac{2}{z_2}C\RobAvgError)}{A-2C\RobAvgError}\geq 1$; and
        \item $\frac{1/2(\frac{1-\MomentumCoeff^{(\GlobalIter)}}{\LearningRate^{(\GlobalIter)} L} - \frac{1}{z_2}C)}{A-2C\RobAvgError}\geq 1$.
    \end{enumerate*} We require the first term
    \begin{align*}
        A-2C\RobAvgError = & \frac{1}{2} \left( 1-4\LearningRate^{(\GlobalIter)} L - 8z_1(1+\LearningRate^{(\GlobalIter)}L) \MomentumCoeff^{(\GlobalIter)^2}  \right) \\
        & - 2 \RobAvgError (1+\LearningRate^{(\GlobalIter)} L + 2 z_1 (1+ \LearningRate^{(\GlobalIter)} L) \MomentumCoeff^{(\GlobalIter)^2}) \\
        = & \frac{5}{2} - 2(1+\RobAvgError)(1+\LearningRate^{(\GlobalIter)} L)(1+2z_1 \MomentumCoeff^{(\GlobalIter)^2})>0,
    \end{align*}
    since $\LearningRate^{(\GlobalIter)} L \leq \frac{1}{24}$, $z_1=\frac{1}{16}$, and $\MomentumCoeff^2 \leq 1$, we obtain 
    \begin{align}
    \label{cons:1}
        \RobAvgError<\frac{1}{15}.
    \end{align}
    For the second term, since $A-2C\RobAvgError>0$ and $z_2=2$, we require $\frac{1}{2}(B-\frac{2}{z_2}C\RobAvgError) - (A-2C\RobAvgError)= -\frac{1}{2} (2A-3C\RobAvgError-B) \geq 0$.
    % \begin{align*}
    %     \frac{1}{2}(B-\frac{2}{z_2}C\RobAvgError) - (A-2C\RobAvgError) & \geq 0 \\
    %     % B-C\RobAvgError - 2A+4C\RobAvgError & \geq 0 \\
    %     2A-3C\RobAvgError-B & \leq 0
    % \end{align*}
    By substituting $A,B$ and $C$ and rearranging the terms, we require
    \begin{align*}
        & 2A\hspace{-1mm}-3C\RobAvgError\hspace{-1mm}-B \hspace{-1mm}
        % = & \left( (1-4\LearningRate^{(\GlobalIter)} L) - 8z_1(1+\LearningRate^{(\GlobalIter)}L) \MomentumCoeff^{(\GlobalIter)^2}  \right) \\
        % &-3 \left( 1+\LearningRate^{(\GlobalIter)} L + 2 z_1 (1+ \LearningRate^{(\GlobalIter)} L) \MomentumCoeff^{(\GlobalIter)^2} \right)\RobAvgError \\
        % &+ \frac{1}{z_1} \!(1+2 \LearningRate^{(\GlobalIter)} L)+\frac{1+5\LearningRate^{(\GlobalIter)} L \! + \! 4 \LearningRate^{(\GlobalIter)^2} L^2}{\LearningRate^{(\GlobalIter)} L} \MomentumCoeff^{(\GlobalIter)^2} \!\! \!- \frac{1}{\LearningRate^{(\GlobalIter)} L}  \\
        % =  &  5-4(1+\LearningRate^{(\GlobalIter)} L)- \frac{1}{2}\MomentumCoeff^{(\GlobalIter)^2}(1+\LearningRate^{(\GlobalIter)}L)  \\
        % & - 3\RobAvgError (1+\frac{1}{8}\MomentumCoeff^{(\GlobalIter)^2})(1+\LearningRate^{(\GlobalIter)}L) \\
        % & + 16 + 32\LearningRate^{(\GlobalIter)} L + \frac{\MomentumCoeff^{(\GlobalIter)^2}}{\LearningRate^{(\GlobalIter)} L} + 5\MomentumCoeff^{(\GlobalIter)^2} + 4\LearningRate^{(\GlobalIter)} L \MomentumCoeff^{(\GlobalIter)^2}- \frac{1}{\LearningRate^{(\GlobalIter)} L} \\
        = 21\hspace{-1mm}-\hspace{-1mm}(4\hspace{-1mm}+\hspace{-1mm}3\RobAvgError)(1+\LearningRate^{(\GlobalIter)}L)\hspace{-1mm}+\hspace{-1mm} 32\LearningRate^{(\GlobalIter)} L - \frac{1}{\LearningRate^{(\GlobalIter)} L} \\
        & \!- \! \MomentumCoeff^{(\GlobalIter)^2} \!\! \left(\! \frac{1}{8}(4+3\RobAvgError)(1+\LearningRate^{(\GlobalIter)}L) \!-\! \frac{1}{\LearningRate^{(\GlobalIter)} L} \!-\! 5 \!-\! 4\LearningRate^{(\GlobalIter)} L \! \right) \! \leq  0.
    \end{align*}
    By looking into the term $\frac{1}{8}(4+3\RobAvgError)(1+\LearningRate^{(\GlobalIter)}L) + \frac{1}{\LearningRate^{(\GlobalIter)} L} + 5 + 4\LearningRate^{(\GlobalIter)} L$, and due to the fact that $0 \leq \RobAvgError \leq 1$, we realize that the term is always positive when $\LearningRate^{(\GlobalIter)} L \leq \frac{1}{24}$. Thus, the maximum of $2A-3C\RobAvgError-B$ is achieved when $\MomentumCoeff^{(\GlobalIter)}=0$.
    Thus, we instead require 
    \begin{align*}
        % & 2A-3C\RobAvgError-B \\
        % = & 21 - (4+3\RobAvgError)(1+\LearningRate^{(\GlobalIter)}L) + 32\LearningRate^{(\GlobalIter)} L - \frac{1}{\LearningRate^{(\GlobalIter)} L} \\
        % & - \MomentumCoeff^{(\GlobalIter)^2} \left( \frac{1}{8}(4+3\RobAvgError)(1+\LearningRate^{(\GlobalIter)}L) - \frac{1}{\LearningRate^{(\GlobalIter)} L} - 5 - 4\LearningRate^{(\GlobalIter)} L \right) \\
        21 - (4+3\RobAvgError)(1+\LearningRate^{(\GlobalIter)}L) + 32\LearningRate^{(\GlobalIter)} L - \frac{1}{\LearningRate^{(\GlobalIter)} L} \leq 0.
    \end{align*}
    Since $\LearningRate^{(\GlobalIter)} L >0 $, this implies $(28-3\RobAvgError)(\LearningRate^{(\GlobalIter)}L)^2 + (17-3\RobAvgError)\LearningRate^{(\GlobalIter)} L  - 1 \leq 0$.
    % \begin{align*}
    %     21\LearningRate^{(\GlobalIter)} L - (4+3\RobAvgError)(1+\LearningRate^{(\GlobalIter)}L)\LearningRate^{(\GlobalIter)} L + 32(\LearningRate^{(\GlobalIter)} L)^2 - 1 & \leq 0 \\
    %     (28-3\RobAvgError)(\LearningRate^{(\GlobalIter)}L)^2 + (17-3\RobAvgError)\LearningRate^{(\GlobalIter)} L  - 1 & \leq 0.
    % \end{align*}
    Since $\LearningRate^{(\GlobalIter)} L \leq \frac{1}{24}$, we require $\RobAvgError \geq -\frac{28}{15}$, which always holds.
    
    The third term requires $\frac{1/2(\frac{1-\MomentumCoeff^{(\GlobalIter)}}{\LearningRate^{(\GlobalIter)} L} - \frac{1}{z_2}C)}{A-2C\RobAvgError}\geq 1$. Using the facts that $0 \leq \MomentumCoeff^{(\GlobalIter)} \leq 1$, $1-\MomentumCoeff^{(\GlobalIter)} = 24\LearningRate^{(\GlobalIter)} L$, $\LearningRate^{(\GlobalIter)} \leq \frac{1}{24L}$, and $A-2C\RobAvgError>0$, we obtain 
    \begin{align*}
        & \!\!\!\! \frac{1-\MomentumCoeff^{(\GlobalIter)}}{\LearningRate^{(\GlobalIter)} L} \!-\! \frac{1}{z_2}C \!=\! \frac{1-\MomentumCoeff^{(\GlobalIter)}}{\LearningRate^{(\GlobalIter)} L} \!-\! \frac{1\!+\!\LearningRate^{(\GlobalIter)} \!L \!+\! 2 z_1\! (1 \!+\! \LearningRate^{(\GlobalIter)} L) \MomentumCoeff^{(\GlobalIter)^2}\!\!}{z_2}  \\
        \geq & \frac{1-\MomentumCoeff^{(\GlobalIter)}}{\LearningRate^{(\GlobalIter)} L} - \frac{1}{2}(1+\frac{1}{24}+32(1+\frac{1}{24})) = 24 -18 =6
    \end{align*}
    Thus, we only require $A-2C\RobAvgError \leq \frac{1}{2}\times 6=3$. Since 
    \begin{align*}
    A-2C\RobAvgError = & \frac{5}{2} - 2(1+\RobAvgError)(1+\LearningRate^{(\GlobalIter)} L)(1+2z_1 \MomentumCoeff^{(\GlobalIter)^2}) \\
    \leq & \frac{5}{2} - 2(1+\RobAvgError) \leq \frac{1}{2},
    \end{align*}
    thus, $A-2C\RobAvgError \leq 3$ always holds.
    By combining the above constraints, we require $\RobAvgError<\frac{1}{15}$.
    % \begin{align*}
    %     \RobAvgError<\frac{1}{15}.
    % \end{align*} 
    Let $\RobAvgError<\frac{1}{15}$, we obtain
    \begin{align*}
        & W^{(\GlobalIter+1)} - W^{(\GlobalIter)} 
        \stackrel{(a)}{\leq}  - 2\LearningRate^{(\GlobalIter)} \mu (A-2C\RobAvgError) \mathbb{E} [ \Loss_\HonestSet(\model^{(\GlobalIter)})-\Loss^{\star} \\
        & \, +\! \frac{\frac{1}{2} (B-\frac{2}{z_1}C\RobAvgError)}{A-2C\RobAvgError}\! \frac{z_1}{L}\! \| \delta^{(\GlobalIter)} \|^2\!\! + \!\frac{\frac{1}{2}(\frac{1-\MomentumCoeff^{(\GlobalIter)}}{\LearningRate^{(\GlobalIter)}L}-\frac{1}{z_2}C)}{A-2C\RobAvgError} \frac{\oringinD \RobAvgCoeff z_2}{L} \Delta\!^{(\GlobalIter)} \!] \\
        &+  \frac{1}{L}(1-\MomentumCoeff^{(\GlobalIter)})^2 \bar{\sigma}^2 + \oringinD \RobAvgCoeff \frac{z_2}{L} (1-\MomentumCoeff^{(\GlobalIter)}) G_{\text{cov}}^2 \\
       & \stackrel{(b)}{\leq}  - 2\LearningRate^{(\GlobalIter)} \mu (A-2C\RobAvgError) \mathbb{E} [ \Loss_\HonestSet(\model^{(\GlobalIter)})-\Loss^{\star} + \frac{z_1}{L} \| \delta^{(\GlobalIter)} \|^2 \\
        & + \frac{\oringinD \RobAvgCoeff z_2}{L} \Delta^{(\GlobalIter)} ] + \frac{1}{L}(1-\MomentumCoeff^{(\GlobalIter)})^2 \bar{\sigma}^2 + \oringinD \RobAvgCoeff \frac{z_2}{L} (1-\MomentumCoeff^{(\GlobalIter)}) G_{\text{cov}}^2 \\
        \stackrel{(c)}{=} & \!\! - \! 2\LearningRate^{(\GlobalIter)} \! \mu( \! A \!-\! 2C\RobAvgError \!) W^{(\GlobalIter)}\!\!\! +\! \frac{( \!1 \! - \! \MomentumCoeff^{(\GlobalIter)}\!)^2 \bar{\sigma}^2 \!\!}{L} \! +\!  \frac{2\oringinD \RobAvgCoeff (\! 1\!-\!\MomentumCoeff^{(\GlobalIter)} \!)G_{\text{cov}}^2 }{L} ,
    \end{align*}
    where $(a)$ and $(b)$ are obtained by satisfying the constraints, and $(c)$ is due to the definition of $W^{(\GlobalIter)}$. Therefore, 
    \begin{align*}
        & V^{(\GlobalIter+1)}-V^{(\GlobalIter)} \\
        \leq & (\hat{t}+1)^2(- 2\LearningRate^{(\GlobalIter)} \mu (A-2C\RobAvgError) W^{(\GlobalIter)}\! +\! \frac{1}{L}(1-\MomentumCoeff^{(\GlobalIter)})^2 \bar{\sigma}^2 \!\\
        & + \! \oringinD \RobAvgCoeff \frac{2}{L} (1-\MomentumCoeff^{(\GlobalIter)}) G_{\text{cov}}^2) +\! \left(2\hat{t}+1\right)W^{(\GlobalIter)} \\
        % = & \left( -(\hat{t}+1)^2 2\LearningRate^{(\GlobalIter)} \mu (A-2C\RobAvgError) + (2\hat{t}+1) \right) W^{(\GlobalIter)} \\
        % & + \frac{(\hat{t}+1)^2}{L}(1-\MomentumCoeff^{(\GlobalIter)})^2 \bar{\sigma}^2 + \oringinD \RobAvgCoeff \frac{2}{L} (1-\MomentumCoeff^{(\GlobalIter)}) (\hat{t}+1)^2 G_{\text{cov}}^2 \\
        \stackrel{(a)}{=} & \left( - \frac{20(\hat{t}+1)^2 (A-2C\RobAvgError)}{\hat{t}}+(2\hat{t}+1) \right) W^{(\GlobalIter)} \\
        + & L a_1^2 (\frac{\hat{t}+1}{\mu \hat{t}})^2 \bar{\sigma}^2 + 2a_1 \oringinD \kappa \frac{(\hat{t}+1)^2}{\mu \hat{t}} G_{\text{cov}}^2 \\
        \stackrel{(b)}{\leq} & L a_1^2 (\frac{\hat{t}+1}{\mu \hat{t}})^2 \bar{\sigma}^2 + 2a_1 \oringinD \kappa \frac{(\hat{t}+1)^2}{\mu \hat{t}} G_{\text{cov}}^2
    \end{align*}
    where $(a)$ is by rearranging the terms and uses $1-\MomentumCoeff^{(\GlobalIter)}=24L\LearningRate^{(\GlobalIter)}$ and $\LearningRate^{(\GlobalIter)}=\frac{10}{\mu(\GlobalIter+a_1 \frac{L}{\mu})}$, and $(b)$ is from the result that $- \frac{20(\hat{t}+1)^2 (A-2C\RobAvgError)}{\hat{t}}+(2\hat{t}+1) \leq 0$ for all possible $\hat{t}=t+a_1\frac{L}{\mu}\geq t+240 \geq 240$ if $\RobAvgError < \frac{3}{125}$. This is because making sure 
    $- \frac{20(\hat{t}+1)^2 (A-2C\RobAvgError)}{\hat{t}}+(2\hat{t}+1) \leq 0,$
    for $\hat{t} \geq 240$, is equivalent to ensuring
    \begin{align*}
        (2-20(A-2C\RobAvgError))\hat{t}^2 + (1-40(A-2C\RobAvgError)) \hat{t} & \\
        - 20(A-2C\RobAvgError) & \leq 0,
    \end{align*}
    for $\hat{t} \geq 240$. Therefore, we require the coefficient of $\hat{\GlobalIter}^2$ to be negative, i.e., 
    $
        A-2C\RobAvgError > \frac{1}{10},
    $
    and the larger root of the function to be greater than $240$, or the coefficient of $\hat{\GlobalIter}^2$ to be zero, i.e., $A-2C\RobAvgError = \frac{1}{10}$, and the remaining terms result in a non-positive value when $\hat{t} \geq 240$. Taking these into account, we require $A-2C\RobAvgError \geq \frac{1}{10}$, i.e., 
    \begin{align}
    \label{cons:2}
        \RobAvgError \leq \frac{3}{125}.
    \end{align}
    Now we obtain two constraints on $\RobAvgError$, i.e., \cref{cons:1,cons:2}, integrating these gives us $\RobAvgError\leq \frac{3}{125}.$
    
    Then, by following the proving techniques in~\cite{allouah2023privacy}, we are able to show that, when $\RobAvgError \leq \frac{3}{125}$,
    \begin{align*}
        \mathbb{E}[\Loss_\HonestSet(\model^{(\TotalGloablIter)})-\Loss^{\star}] \leq \frac{4a_1 \oringinD \RobAvgCoeff G_\mathrm{cov}^2}{\mu}+\frac{2a_1^2L\bar{\sigma}^2}{\mu^2\TotalGloablIter}+\frac{2a_1^2L^2\Loss_0}{\mu^2\TotalGloablIter^2},
    \end{align*}
    where $\Loss_0 \defeq \Loss_\HonestSet(\model^{(0)})- \Loss^{\star}$.
% \end{proof}

\subsubsection{Non-convex}
Then, we prove the convergence for the non-convex case.
% \begin{proof}
We set $\LearningRate= \min\{ \frac{1}{24L}, \frac{1}{8\bar{\sigma}} \sqrt{\frac{\Loss_\HonestSet(\model^{(0)})- \Loss^{\star}}{8 L \TotalGloablIter}} \}$, $\MomentumCoeff = 1-24L \LearningRate^{(\GlobalIter)}$, where $a_1=240$. We have $\LearningRate^{(\GlobalIter)} \leq \frac{1}{24L}$.
Suppose $\MomentumCoeff^{(\GlobalIter)}=\MomentumCoeff$ and $\LearningRate^{(\GlobalIter)}=\LearningRate$, from \cref{eq:Wdiff} we have
\begin{align*}
    & W^{(\GlobalIter+1)} - W^{(\GlobalIter)} \\
    \leq & \!-\!\LearningRate (\!A\! - \!2C\RobAvgError\!)\! \cdot \! \mathbb{E}\!\! \left[ \! \| \!\nabla\!\Loss_\HonestSet\!(\model^{(\!\GlobalIter)}\!)\! \|^2 \! \right] \!\!-\!\!z_1 \!\LearningRate (\!B\! -  \!\frac{2C \RobAvgError}{z_1} \! ) \!\!\cdot \! \mathbb{E} \!\left[\! \| \delta^{(\!\GlobalIter)} \! \|^2 \!\right] \\
    & \!\!\!- \!\oringinD \RobAvgCoeff z_2 \LearningRate \!\left(\!\! \frac{1-\MomentumCoeff}{\LearningRate L} \!-\! \frac{C}{z_2}\! \right)\! \!\cdot \!\mathbb{E}\!\! \left[ \!\Delta^{\!(\GlobalIter)}\!\right] \!\!+\! \frac{\!(1-\MomentumCoeff)^2\!\!}{L} \bar{\sigma}^2 \!\!\!+\! \oringinD \RobAvgCoeff \frac{z_2}{L} \!(1\!-\!\MomentumCoeff) G_{\text{cov}}^2.
\end{align*}

From the proof for the strongly-convex setting, we know that $\frac{1}{2}(B-\frac{2}{z_2}C\RobAvgError) \geq (A-2C\RobAvgError)  > 0$ and $1/2(\frac{1-\MomentumCoeff^{(\GlobalIter)}}{\LearningRate^{(\GlobalIter)} L} - \frac{1}{z_2}C) \geq A-2C\RobAvgError > 0$ always hold if $\RobAvgError<\frac{1}{15}$, thus
\begin{align*}
     W^{(\GlobalIter+1)} - W^{(\GlobalIter)}
    \leq & -\LearningRate (A- 2C\RobAvgError) \cdot \mathbb{E} \left[  \| \nabla\Loss_\HonestSet(\model^{(\GlobalIter)}) \|^2\right] \\
        & + \frac{1}{L}(1-\MomentumCoeff)^2 \bar{\sigma}^2 + \oringinD \RobAvgCoeff \frac{z_2}{L} (1-\MomentumCoeff) G_{\text{cov}}^2,
\end{align*}
with $\RobAvgError<\frac{1}{15}$.
By rearranging terms, we have
\begin{align*}
    & \LearningRate (A- 2C\RobAvgError) \cdot \mathbb{E} \left[  \| \nabla\Loss_\HonestSet(\model^{(\GlobalIter)}) \|^2\right]  \\
    \leq & W^{(\GlobalIter)} - W^{(\GlobalIter+1)}
        + \frac{1}{L}(1-\MomentumCoeff)^2 \bar{\sigma}^2 + \oringinD \RobAvgCoeff \frac{z_2}{L} (1-\MomentumCoeff) G_{\text{cov}}^2 .
\end{align*}
Since $\MomentumCoeff=1-24\LearningRate L$, after averaging over $\GlobalIter \in \{ 0, \cdots, \TotalGloablIter\hspace{-1mm}-\hspace{-1mm}1 \}$,
\begin{align*}
    & \frac{1}{\TotalGloablIter} \sum_{\GlobalIter=0}^{\TotalGloablIter-1} \mathbb{E} \left[  \| \nabla\Loss_\HonestSet(\model^{(\GlobalIter)}) \|^2\right]  \\
    \leq & \frac{ W^{(0)} - W^{(\TotalGloablIter)}}{\LearningRate (A- 2C\RobAvgError) \TotalGloablIter} \!+\! \frac{(1-\MomentumCoeff)^2 \bar{\sigma}^2}{L\LearningRate (A- 2C\RobAvgError)} \!+\! \frac{\oringinD \RobAvgCoeff z_2(1-\MomentumCoeff) G_{\text{cov}}^2}{L\LearningRate (A- 2C\RobAvgError)}   \\
    \stackrel{(a)}{=} & \frac{ W^{(0)} - W^{(\TotalGloablIter)} }{\LearningRate (A- 2C\RobAvgError) \TotalGloablIter}  + \frac{24^2 \LearningRate L}{A- 2C\RobAvgError} \bar{\sigma}^2 + \frac{24 \oringinD \RobAvgCoeff z_2}{A- 2C\RobAvgError}  G_{\text{cov}}^2,
\end{align*}
where $(a)$ is by plugging in $1- \MomentumCoeff=24\LearningRate L$. By following the argument in~\cite{allouah2023privacy} and initializing $\Momentum_i^{(0)}=0$, we know $\Delta^{(0)}=0$, $\| \delta^{(0)} \|^2=\| \nabla \Loss_{\mathcal{H}} (\model^{(0)}) \|^2$, and $\| \nabla \Loss_{\mathcal{H}} (\model^{(0)}) \|^2 \leq 2L(\Loss_{\mathcal{H}} (\model^{(0)}) - \Loss^{\star})$ since $\Loss_{\mathcal{H}}$ is $L$-smooth. Thus, by plugging in $z_1=\frac{1}{16}$, we obtain
\begin{align*}
   & W^{(0)}\hspace{-1mm}- \hspace{-1mm}W^{(\TotalGloablIter)} \hspace{-1mm}\leq W^{(0)} \hspace{-1mm}
    \leq \Loss_\HonestSet(\model^{(0)})\hspace{-1mm}-\hspace{-1mm}\Loss^{\star} \!+\! \frac{z_1}{L} \| \delta^{(0)} \|^2 \!+\! \oringinD\RobAvgCoeff \cdot\hspace{-1mm} \frac{z_2}{L}\Delta^{(0)} \\
    & = \Loss_\HonestSet(\model^{(0)})\hspace{-1mm}- \hspace{-1mm}\Loss^{\star} \!+\! \frac{z_1}{L} \| \nabla \Loss_{\mathcal{H}} (\model^{(0)}) \|^2 
    % & \leq \Loss_\HonestSet(\model^{(0)})- \Loss^{\star} \!+\! \frac{z_1 2L}{L} (\Loss_{\mathcal{H}} (\model^{(0)}) - \Loss^{\star}) \\
    \leq \frac{9}{8} \left( \Loss_\HonestSet(\model^{(0)})- \Loss^{\star} \right).
\end{align*}
Let $\Loss_0 \defeq \Loss_\HonestSet(\model^{(0)})- \Loss^{\star}$. By plugging this bound, we obtain
\begin{align*}
    & \frac{1}{\TotalGloablIter} \sum_{\GlobalIter=0}^{\TotalGloablIter-1} \mathbb{E} \left[  \| \nabla\Loss_\HonestSet(\model^{(\GlobalIter)}) \|^2\right]  \\
    \leq & \frac{9\Loss_0}{8T\LearningRate (A- 2C\RobAvgError)} + \frac{24^2 \LearningRate L}{A- 2C\RobAvgError} \bar{\sigma}^2 + \frac{\frac{3}{2} \oringinD \RobAvgCoeff }{A- 2C\RobAvgError}  G_{\text{cov}}^2.
\end{align*}
We notice that only the first two terms depend on $\LearningRate$, and the optimal balance is obtained by minimizing $\frac{9 \Loss_0}{8\TotalGloablIter\LearningRate (A- 2C\RobAvgError)} + \frac{24^2 \LearningRate L}{A- 2C\RobAvgError} \bar{\sigma}^2$. Thus, we choose $\LearningRate=\sqrt{\frac{\Loss_0}{8^3 LT \bar{\sigma}^2}}$.

 Since we set $\LearningRate= \min\{ \frac{1}{24L}, \frac{1}{8\bar{\sigma}} \sqrt{\frac{\Loss_0}{8 L \TotalGloablIter}} \}$, thus, we obtain
\begin{align*}
    & \frac{1}{\TotalGloablIter} \sum_{\GlobalIter=0}^{\TotalGloablIter-1} \mathbb{E} \left[  \| \nabla\Loss_\HonestSet(\model^{(\GlobalIter)}) \|^2\right]  \\
        % \leq & \frac{9 \Loss_0}{8\TotalGloablIter (A- 2C\RobAvgError) \cdot \LearningRate} + \frac{24^2 L \bar{\sigma}^2 \cdot \LearningRate}{A- 2C\RobAvgError} + \frac{\frac{3}{2} \oringinD \RobAvgCoeff \cdot G_{\text{cov}}^2 }{A- 2C\RobAvgError} \\
        % \stackrel{(a)}{\leq} & \frac{9 \Loss_0}{8\TotalGloablIter (A- 2C\RobAvgError)}  \cdot \left( 24L + 8\bar{\sigma} \frac{8L \TotalGloablIter}{\Loss_0} \right) \\
        % & + \frac{24^2 L \bar{\sigma}^2}{A- 2C\RobAvgError}\cdot \frac{1}{8\bar{\sigma}} \sqrt{\frac{\Loss_0}{8 L \TotalGloablIter }} + \frac{\frac{3}{2} \oringinD \RobAvgCoeff }{A- 2C\RobAvgError}  G_{\text{cov}}^2 \\
        \stackrel{(a)}{\leq} & \frac{1}{A- 2C\RobAvgError} ( \frac{27L \Loss_0}{\TotalGloablIter} + \frac{36 \bar{\sigma} \sqrt{2L \Loss_0}}{\sqrt{\TotalGloablIter}} + \frac{3}{2} \oringinD \RobAvgCoeff G_{\text{cov}}^2 ),
\end{align*}
where $(a)$ uses the facts that $\LearningRate \leq \frac{1}{8\bar{\sigma}} \sqrt{\frac{\Loss_0}{8 L \TotalGloablIter}}$, and $\frac{1}{\LearningRate} = \max \{ 24L,  8\bar{\sigma} \frac{8L \TotalGloablIter}{\Loss_0}\} \leq 24L + 8\bar{\sigma} \frac{8L \TotalGloablIter}{\Loss_0}$, and $0< A- 2C\RobAvgError \leq \frac{1}{2}$.

Since $\tilde{\model}$ is randomly chosen from $\{ \model^{0},\cdots, \model^{(\GlobalIter)} \}$, we have $\mathbb{E} \left[ \nabla \Loss_{\mathcal{H}}(\tilde{\model}) \right]=\frac{1}{\TotalGloablIter} \sum_{\GlobalIter=0}^{\TotalGloablIter-1} \mathbb{E} \left[  \| \nabla\Loss_\HonestSet(\model^{(\GlobalIter)}) \|^2\right]$. Hence, if $0<\RobAvgError<\frac{1}{15}$, we obtain the above bound.

\subsection{Proof of \cref{the:scheme}}
\label{app:proof}
First, we present some important preliminaries.
% required by the privacy and convergence analysis.
\begin{proposition}[Composition of RDP~\cite{mironov2017renyi}]
    Let $f$ be an $(\RDPalpha,\RDPepsilon)$-RDP mechanism that takes a dataset as input, and $f'$ be an $(\RDPalpha,\RDPepsilon')$-RDP mechanism that takes both the dataset and the output of $f$ as input.  Then, their composition satisfies $(\RDPalpha,\RDPepsilon+\RDPepsilon')$-RDP.
\end{proposition}

\begin{proposition}[$l_2$-sensitivity~\cite{dwork2014algorithmic}]
    For a given randomized mechanism $f$, the $l_2$-sensitivity of $f$ is defined as $\Delta_2\defeq\max \| f(\Dataset)-f(\Dataset') \|$, where $\Dataset$ and $\Dataset'$ are adjacent datasets.
\end{proposition}

\begin{proposition} [RDP for Gaussian Mechanisms~\cite{mironov2017renyi}]
\label{pro:rdp_gaussian}
    Let $f$ be a mechanism with $l_2$-sensitivity at most $\Delta_2$. Then the Gaussian mechanism $f+\mathcal{N}(0,\DPsigma^2I_d)$ satisfies $(\RDPalpha,\frac{\Delta_2^2}{2\DPsigma^2}\RDPalpha)$-RDP.
\end{proposition}

Subsampling is a privacy amplification method and is a popular mechanism in machine learning, e.g., SGD. Since we perform subsampling before computing gradients and adding noise, we must analyze RDP under subsampled mechanisms.  
% Since we use subsampling methods before computing gradients and applying noise, we need to consider RDP for subsampled mechanisms. 
Different subsampling methods give different privacy amplification guarantees~\cite{zhu2019poission, wang2019subsampled}. 
For the clarity in the differential privacy analysis, we use subsampling without replacement as our sampling method. Other sampling methods, such as subsampling with replacement or Poisson subsampling, can also be used, but they require separate analyses. We use the following proposition.

Post-processing is another important property guaranteed by differential privacy, including RDP. Let $f$ be a mechanism that satisfies $(\RDPalpha, \RDPepsilon)$-RDP, and let $g$ be a randomized mapping. Then RDP is preserved under post-processing~\cite{mironov2017renyi},
% Let $f$ be an $(\RDPalpha,\RDPepsilon)$-RDP, $g$ be a randomized mapping, RDP is preserved,
i.e., $g(f(\cdot))$ also satisfies $(\RDPalpha,\RDPepsilon)$-RDP.

Now we are ready to prove \cref{the:scheme}

% \begin{proof}
% Here we prove the DP guarantee, convergence guarantee, and the communication efficiency of \ourscheme. 

\noindent\textbf{Privacy.}
In each iteration,  we apply Gaussian noise to the gradients, which are directly computed from the corresponding datasets. All operations- including  momentum computation, compression, robust aggregation, decompression, and model update- are considered post-processing of the perturbed gradients. Thanks to the post-processing and the composition property of differential privacy,
it suffices to compute the DP guarantee up to the noise addition step for a single iteration, 
% we compute the DP guarantee up to the noise addition step as the DP guarantee for a single iteration,
and then apply composition to obtain the overal DP guarantee for the full algorithm.

Consider two adjacent datasets $\Dataset$ and $\Dataset'$. Two minibatches $\Minibatch_i^{(\GlobalIter)} \subseteq \Dataset$ and $\Minibatch_i'^{(\GlobalIter)} \subseteq \Dataset'$ are sampled without replacement, such that $\Minibatch_i^{(\GlobalIter)}$ and $\Minibatch_i'^{(\GlobalIter)}$ differ by at most one element. Let
% We denote the gradient of the average of the clipped per-sample gradients 
$\gradient_i^{(\GlobalIter)}$ and $\gradient_i'^{(\GlobalIter)}$, 
denote the averaged clipped gradients computed from the respective minibatches.   
From the proof for Theorem C$.1$ in~\cite{allouah2023privacy}, we know that 
% \begin{align}
% \label{eq:sensitivity_gradient}
   $ \| \gradient_i^{(\GlobalIter)} - \gradient_i'^{(\GlobalIter)} \| \leq \frac{2\ClipNorm}{\MinibatchSize} = \Delta_2,$
% \end{align}
where $\gradient_i^{(\GlobalIter)}= \frac{1}{\MinibatchSize} \sum_{x \in \Minibatch_i^{(t)}} \text{Clip}(\nabla \SampleLoss(\model^{(\GlobalIter)}; \bm{x}, y), \ClipNorm)$ and $\gradient_i'^{(\GlobalIter)}= \frac{1}{\MinibatchSize} \sum_{x \in \Minibatch_i'^{(t)}} \text{Clip}(\nabla \SampleLoss(\model^{(\GlobalIter)}; \bm{x}, y), \ClipNorm)$.
Therefore, by \cref{pro:rdp_gaussian}, a single iteration of \ourscheme is $(\RDPalpha,(\frac{2\ClipNorm}{\MinibatchSize})^2\frac{\RDPalpha}{2\DPsigma^2})$-RDP. 
Following the ideas from Theorem C$.1$ and $4.1$ in~\cite{allouah2023privacy}, we can prove that, there exists a constant $\chi>0$ such that, for sufficiently small batch size $\MinibatchSize$ ( i.e., $\frac{\MinibatchSize}{\DatasetSize}$ is sufficiently small), when $\DPsigma \geq \chi \frac{2\ClipNorm}{\MinibatchSize} \max \left\{ 1,\frac{\MinibatchSize\sqrt{\TotalGloablIter \log(1/\DPdelta)}}{\DatasetSize \DPepsilon} \right\}$, \ourscheme satisfies  $(\DPepsilon,\DPdelta)$-DP.

\noindent\textbf{Convergence.} 
From the proof of \cref{pro:JLwithRobAvg_deterministic} and \cref{pro:JLwithRobAvg}, we know that, given $\AGGCall$ being $(\Byzantine,\RobAvgCoeff)$-robust averaging, w.p. $1-\frac{1}{|\mathcal{V}|}$, the composition $\compressCall \circ \AGGCall \circ \decompressCall$ provides a $(\Byzantine,\RobAvgCoeff',\RobAvgError)$-robust averaging, with $\RobAvgCoeff'=(1+\JLepsilon)^2 \RobAvgCoeff$ and $\RobAvgError=\JLepsilon^2$. By plugging this result in \cref{pro:composition}, we obtain the conditions on the additional error term $\RobAvgError$ for the strongly convex case and the non-convex case.

\noindent\textbf{Communication.} Due to the use of the Johnson-Lindenstrauss transform as the compression method $\compressCall$, with the compression rate $\frac{\oringinD}{\compressD}$, the noisy momentum vectors  $\Momentum_i \in \mathbb{R}^{\oringinD}$, for $i \in [\Totalclient]$, are compressed to $\compMomentum_i=\JLmatrix\Momentum_i$, where $\compMomentum_i \in \mathbb{R}^{\compressD}$, before being sent to the federator for robust aggregation. The federator then returns a compressed aggregated vector, also in $\mathbb{R}^{\compressD}$, to the clients. Hence, the communication cost is reduced proportionally to the compression rate in both the uplink and downlink directions.

\end{document}